\documentclass[nohyperref]{article}

\usepackage{hyperref}

\usepackage{microtype}
\usepackage{graphicx}
\usepackage{booktabs} %
\usepackage{amsmath, amsfonts}
\usepackage{amsthm}
\usepackage{bm}
\usepackage{color}
\usepackage{amssymb}
\usepackage{mathtools}

\usepackage{subcaption}
\usepackage{pifont}
\usepackage{xcolor}
\newcommand{\cmark}{\textcolor{green!80!black}{\ding{51}}}
\newcommand{\xmark}{\textcolor{red}{\ding{55}}}
\usepackage{xr}
\usepackage{navigator}
\usepackage{url}
\usepackage{soul}
\usepackage{balance}

\usepackage[accepted]{icml2023}

\theoremstyle{plain}
\newtheorem{theorem}{Theorem}[section]
\newtheorem{proposition}[theorem]{Proposition}
\newtheorem{lemma}[theorem]{Lemma}

\theoremstyle{definition}
\newtheorem{definition}[theorem]{Definition}
\newtheorem{assumption}[theorem]{Assumption}
\theoremstyle{remark}

\DeclareMathOperator*{\argmin}{arg\,min}
\DeclareMathOperator{\thetab}{\bm\theta}

\graphicspath{ {./} }
\usepackage{thmtools}
\usepackage{thm-restate}
\usepackage{dcolumn}
\usepackage{multirow}
\usepackage{enumitem}

\newcolumntype{d}[1]{D{.}{.}{4}}%
\newcommand{\subhead}[1]{\multicolumn{1}{c}{#1}}%

\DeclareMathOperator{\x}{\textbf{x}}
\DeclareMathOperator{\y}{\textbf{y}}

\DeclareMathOperator{\w}{\textbf{w}}

\DeclareMathOperator{\p}{\textbf{p}}
\DeclareMathOperator{\q}{\textbf{q}}
\DeclareMathOperator{\s}{\textbf{s}}
\DeclareMathOperator{\z}{\textbf{z}}
\DeclareMathOperator{\h}{\textbf{h}}
\DeclareMathOperator{\lv}{\bm \ell}

\DeclareMathOperator{\B}{\mathcal{B}}
\DeclareMathOperator{\Pm}{\textbf{P}}
\DeclareMathOperator{\Sm}{\textbf{S}}

\DeclareMathOperator{\U}{\textbf{U}}

\DeclareMathOperator{\V}{\textbf{V}}

\DeclareMathOperator{\X}{\textbf{X}}

\DeclareMathOperator{\Q}{\textbf{Q}}

\DeclareMathOperator{\zero}{\textbf{0}}

\DeclareMathOperator{\true}{\diamond}

\newcommand{\rev}[1]{{#1}}

\icmltitlerunning{Communication-Efficient Feature Selection in Vertical Federated Learning}

\begin{document}

\twocolumn[
\icmltitle{LESS-VFL: Communication-Efficient Feature Selection for \\ Vertical Federated Learning}

\begin{icmlauthorlist}
\icmlauthor{Timothy Castiglia}{rpi}
\icmlauthor{Yi Zhou}{ibm}
\icmlauthor{Shiqiang Wang}{ibm}
\icmlauthor{Swanand Kadhe}{ibm}
\icmlauthor{Nathalie Baracaldo}{ibm}
\icmlauthor{Stacy Patterson}{rpi}
\end{icmlauthorlist}

\icmlaffiliation{rpi}{Rensselaer Polytechnic Institute}
\icmlaffiliation{ibm}{IBM Research}

\icmlcorrespondingauthor{Timothy Castiglia}{castigliatim@gmail.com}
\icmlkeywords{Machine Learning, Federated Learning, Feature Selection, Distributed Systems}

\vskip 0.3in
]
\printAffiliationsAndNotice{} %

\begin{abstract}
We propose LESS-VFL, a communication-efficient feature selection method for distributed systems with vertically partitioned data. We consider a system of a server and several parties with local datasets that share a sample ID space but have different feature sets. The parties wish to collaboratively train a model for a prediction task. As part of the training, the parties wish to remove unimportant features in the system to improve generalization, efficiency, and explainability. In LESS-VFL, after a short pre-training period, the server optimizes its part of the global model to determine the relevant outputs from party models. This information is shared with the parties to then allow local feature selection without communication. We analytically prove that LESS-VFL removes spurious features from model training. We provide extensive empirical evidence that LESS-VFL can achieve high accuracy and remove spurious features at a fraction of the communication cost of other feature selection approaches.
\end{abstract}

\section{Introduction} \label{intro.sec}

Federated learning has recently become of interest to the research community, 
and has shown promise in several application areas, such as healthcare, smart transportation, 
and predictive energy systems~\cite{Sun2019-ki, kairouz2019advances, zhou2021survey}.
Federated learning algorithms support distributed model training
among parties without the need to directly share local private data.

Vertical Federated Learning (VFL), an important class of federated learning algorithms, has received a significant
amount of attention lately~\citep{DBLP:journals/tist/YangLCT19, verticalAutoencoders, castiglia2022compressed}.
VFL works consider the case where parties store a shared sample ID space,
but different private feature sets. For example, a healthcare provider, 
a wearable technology company, and an insurance company wish to collaboratively
train a model for disease prediction. The parties have information on
the same individuals (sample ID space), but each party stores different health information (feature space).
In VFL, parties typically use local feature extractor models, such as deep neural networks, 
to produce low-dimensional \emph{embeddings} of local feature sets~\citep{FDML, SplitNN}.
The server takes embeddings as input to a fusion model for predictions.
We provide an example VFL model in Figure~\ref{model.fig}.

\begin{figure}[t]
    \centering
    \includegraphics[width=0.4\textwidth]{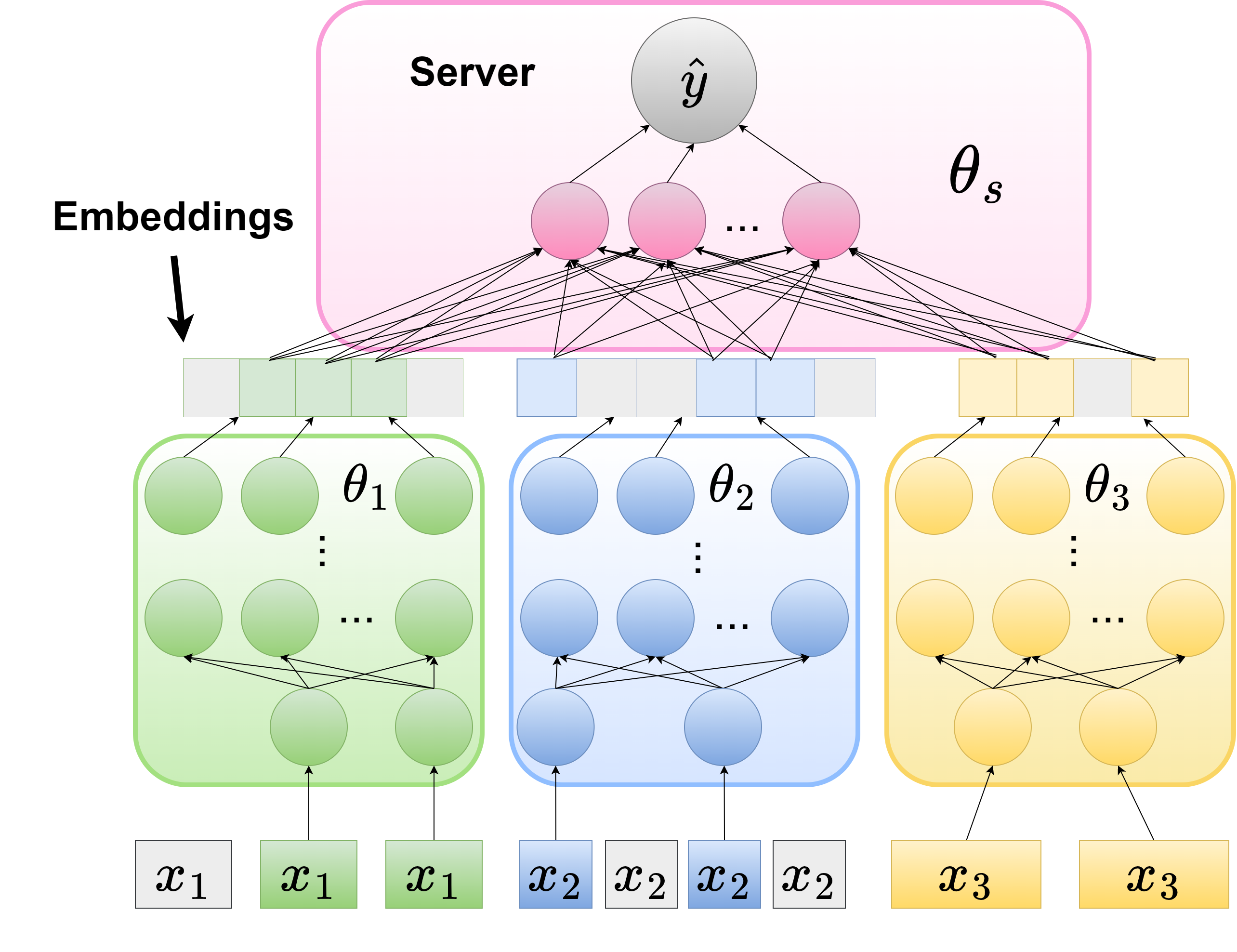}
    \caption{Example VFL model architecture. Non-significant features and embedding components (in gray) are removed after training with LESS-VFL.}
    \vspace{-0.5em}
    \label{model.fig}
\end{figure}

Feature selection is an important part of machine learning tasks.
Often, datasets contain many spurious features that do not relate
to the current prediction task.
For example, health care providers may train models using
electronic medical records (EMRs), 
which contain clinical documents, results from routine visits, and 
many features that may be irrelevant to disease diagnosis~\cite{CaninoSGTZV16}.
Failing to remove spurious features can have drastic effects on generalization.
In Figure~\ref{spur.fig}, we compare the test accuracy of VFL training 
with the original dataset against training with the dataset and an 
additional set of Gaussian noise features.
Simply adding in these spurious features causes the test accuracy of VFL to fall drastically.
In addition to improving model generalization, feature selection is often used for model explainability~\cite{clinciu2019survey}.

\begin{figure}[t]
    \centering
    \includegraphics[width=0.45\textwidth]{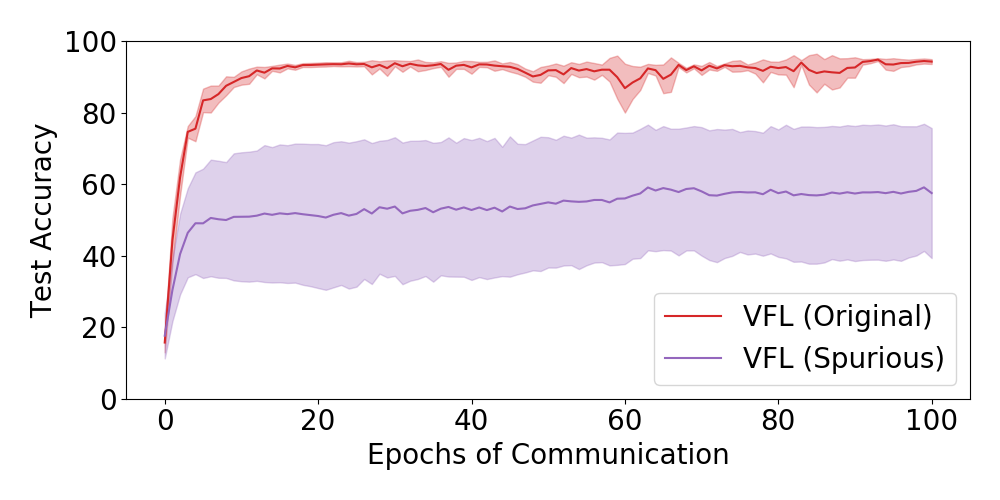}
    \caption{Test accuracy on the Activity dataset (details in Section~\ref{exp.sec}). VFL (Original) denotes the test accuracy of running Algorithm~\ref{vfl.alg} with the original dataset. VFL (Spurious) denotes the test accuracy on the dataset with spurious features added in. The solid line is the average of $5$ runs and the shaded region represents the standard deviation.}
    \label{spur.fig}
\end{figure}

Most centralized feature selection algorithms cannot be directly
translated to VFL because it either requires direct data sharing or 
is communication inefficient.
Parties in VFL may be globally distributed, leading communication to be
expensive in time, money, and resources. 
Thus it is important to design communication-efficient VFL algorithms. 
Although a few works propose VFL feature selection algorithms (summarized in Table~\ref{related.table}), 
no work formally analyzes the
feature selection problem in VFL and creates a method that
provably removes spurious features.
In this work, we seek to answer the following:
\emph{Can we design a communication-efficient VFL feature selection 
algorithm and formally verify that it removes spurious features and achieves high accuracy?}

\textbf{Related Work.} 
Feature selection algorithms tend to broadly fit into three categories:
filter, wrapper, and embedded methods~\cite{ChandrashekarS14}.
Filter methods %
use statistical metrics of the data to determine feature importance a priori to training a model. 
These methods require direct access to features to calculate the metrics and
cannot be directly applied to the VFL setting without sharing raw data.
Wrapper methods %
typically 
involve retraining a model several times to determine the importance of different feature subsets.
This is impractical for the VFL setting where model training requires a large amount of communication between parties and the server.
Embedded methods %
involve 
training a model while simultaneously determining the importance of all 
features.
These methods may fully train a model before performing feature selection, 
or gradually remove unimportant features during training.
Embedded methods that remove unimportant features during training
seem to be a good fit for VFL, however
they must be adapted to support distributed training and 
keep communication overhead low.

\rev{There have been a few works that propose embedded VFL feature selection methods~\cite{MultiVFL2020, ChenZGYFW021, HouSFY22, ZhangLHCZ22, zhang2022embedded, ChenDLWH22, Feng22Group, fedsdgfs}.}
However, most of these methods lack support
for deep neural networks or require a fully trained model 
to begin feature selection (see Table~\ref{related.table}).
Critically, none of these works provide theoretical evidence that 
spurious features are removed with their proposed methods,
only providing empirical evidence. 
An important open problem is how to formalize the feature selection 
problem in the VFL setting and provide a theoretical framework for 
proving that unimportant features are removed.

\begin{table}
    \caption{VFL feature selection algorithms.}
\label{related.table}
\vskip 0.1in
\small
\centering
\resizebox{0.49\textwidth}{!}{
\begin{tabular}{lccc}
    \toprule 
    \multirow{2}{*}{\begin{tabular}{@{}c@{}} \textbf{VFL Feature} \\ \textbf{Selection Algorithm} \end{tabular} } & 
    \multirow{2}{*}{\begin{tabular}{@{}c@{}} \textbf{Supports neural} \\ \textbf{networks} \end{tabular} } &
    \multirow{2}{*}{\begin{tabular}{@{}c@{}} \textbf{Features selected} \\ \textbf{during training} \end{tabular} } &
    \multirow{2}{*}{\begin{tabular}{@{}c@{}} \textbf{Provably removes} \\ \textbf{spurious features} \end{tabular} } \\
    & & & \\
    \midrule
    MMVFL~\cite{MultiVFL2020}      & \xmark & \cmark & \xmark \\
    Fed-EINI~\cite{ChenZGYFW021}      & \xmark & \cmark & \xmark \\
    \citet{HouSFY22}          & \xmark & \cmark & \xmark \\
    \citet{zhang2022embedded} & \xmark & \cmark & \xmark \\
    \citet{ZhangLHCZ22}       & \cmark & \xmark & \xmark \\
    EVFL~\cite{ChenDLWH22}        & \cmark & \xmark & \xmark \\
    VFLFS~\cite{Feng22Group}    & \cmark & \cmark & \xmark \\
    \rev{FedSDG-FS~\cite{fedsdgfs}}    & \cmark & \cmark & \xmark \\
    \midrule
    LESS-VFL (ours)             & \cmark & \cmark & \cmark \\
    \bottomrule
\end{tabular}}
\end{table}

\textbf{Contributions.} 
In this work, we formalize the VFL feature selection problem and
propose Local communication-Efficient group laSSo for Vertical Federated Learning (LESS-VFL),
an embedded feature selection method for VFL that provably removes
spurious features in a communication-efficient manner.
Our method utilizes group lasso regularization~\cite{ZhaoHW15, ZhangWSZP20, 0010ZWPP21} in a novel way that reduces the amount of communication between parties.
After a short pre-training period, the server determines a set of ``significant" embedding components from each party.
Using this information, each party performs feature selection utilizing group lasso locally without communication.
Although it has been proven that a centralized implementation of group lasso removes spurious features~\cite{DinhH20Neurips}, it is
not obvious that our method can provide similar guarantees in VFL.
We prove in our analysis that the parties asymptotically
solve the feature selection problem 
in terms of the sample size given the set of
significant embedding components.
In our experiments, we compare LESS-VFL to applying group lasso 
regularization directly to VFL and find that our method can greatly reduce the 
communication cost of feature selection.

We summarize our contributions:
\begin{itemize}[leftmargin=*, nosep]
    \item We formalize the feature selection problem for VFL in Section~\ref{problem.sec}.
    \item We propose a three-stage approach, namely LESS-VFL, along with a practical implementation in Section~\ref{alg.sec}.
    \item We prove analytically that LESS-VFL removes spurious features and achieves high accuracy in Section~\ref{analysis.sec}.
    \item We provide empirical evidence that LESS-VFL achieves high accuracy at a low communication cost in Section~\ref{exp.sec}.
\end{itemize}

\section{Problem Formulation and Background} \label{problem.sec}

We consider a system with $M$ parties and a server.
Each party $m$ stores $d_m$ features for $N$ training samples. 
We denote party $m$'s dataset as $\X_m \in \mathbb{R}^{N \times d_m}$.
We let the $i$-th sample in $\X_m$ be denoted as $\x_m^{(i)}$.
We assume that each party's dataset is aligned, i.e., $\x_m^{(i)}$
and $\x_j^{(i)}$ for all parties $m \neq j$ are different features
for data sample $i$.
We let a combined data sample be denoted as $\x^{(i)} = [\x_1^{(i)}, \ldots, \x_M^{(i)}]$
and let $\mathcal{X}$ be the set of all possible values for a sample $\x^{(i)}$.
We denote the combined dataset as $\X = [\X_1,\ldots,\X_M]$.
We assume the server stores the training labels $\y$.
Note that any party can play the role of server if it stores the labels.

Each party trains a local model $\h_m(\cdot)$ with parameters 
$\thetab_m$, and the server trains a server model $\h_s(\cdot)$
with parameters $\thetab_s$.
The output of the party models are called \emph{embeddings}. 
All party embeddings act as input to the server fusion model $\h_s(\cdot)$. 
The full VFL model $f(\cdot)$ is defined as follows:
\begin{align*}
    f(\bm\Theta ; \x^{(i)}) \coloneqq
    \h_s(\thetab_s, \h_1(\thetab_1 ; \x_1^{(i)}), \ldots, \h_M(\thetab_M ; \x_M^{(i)})).
\end{align*}
where $\bm\Theta = [\thetab_s^{\top}, \thetab_1^{\top}, \ldots, \thetab_M^{\top}]^{\top}$.

We formalize the feature selection problem.
Recall from Section~\ref{intro.sec} that not all input features may be 
significant to the current prediction task.
We formalize this notion of significance for any given input $\w$ and model $u(\thetab)$, which extends Definition 2.2 in \citet{DinhH20Neurips} to VFL systems.
\begin{definition}\label{sig.def} 
For a given model $u$ with parameters $\thetab$, 
let $\w^j$ be the $j$-th input to the model $u(\thetab ; \w)$, and $g^j(\w,s)$ be a function that replaces $\w^j$ with value $s$.
    The input $\w^j$ is \emph{non-significant} in this model $u(\thetab; \w)$
    iff $u(\thetab;\w) = u(\thetab; g^j(\w,s))$ for all $s \in \mathbb{R}$. 
    Otherwise, $\w^j$ is \emph{significant}.
\end{definition}

We want to emphasize that in Definition~\ref{sig.def}, the set of significant inputs is dependent on the model parameters $\thetab$. 
Throughout this paper, we apply the notion of significance in 
Definition~\ref{sig.def} to the following two scenarios specifically:
\begin{enumerate}[leftmargin=*, nosep]
\item When the input is the set of \textit{training features} and the model is the classifier, i.e.,  $\w = \x$, $u(\cdot) = f(\cdot)$, and $\thetab = \bm\Theta$ in Definition~\ref{sig.def};
\item When the input is the set of \textit{embedding components} generated based on parties' local models and the model is the server model, i.e. $\w = [\h_1(\cdot) ; \ldots ; \h_M(\cdot)]$, ${u(\cdot) = \h_s(\cdot)}$, and $\thetab = \thetab_s$.
\end{enumerate}

We define $\bm\Theta^{\true} = [(\thetab_s^{\true})^{\top}, (\thetab_1^{\true})^{\top}, \ldots, (\thetab_M^{\true})^{\top}]^{\top}$ as the \emph{generating model} that generated the training labels:
${y^{(i)} = f(\bm\Theta^{\true} ; \x^{(i)}) + \epsilon^{(i)}}$ where $\x^{(i)}$
are drawn from a distribution $\mathcal{P}_{\X,y}$ and 
$\epsilon \sim \mathcal{N}(0, \sigma^2)$ (formal definition in Section~\ref{analysis.sec}). 
Our main goal in feature selection is to determine the set of 
significant features for $\bm\Theta^{\true}$.
We let the set of significant features for a party $m$'s 
generating model $\thetab_m^{\true}$ be $\s_m$ and the set of 
non-significant features be $\z_m$ for any data sample $\x_m$.
We can consider the input layer weights that correspond
to the significant and non-significant features:
$$\x_m = (\s_m, \z_m)~\text{and}~\pi(\thetab_m^{\true}) = (\U_m, \V_m)$$
where $\pi(\thetab)$ extracts the input weights in a model $\thetab$, 
and $\U_m$ and $\V_m$ are the input layer weights for the 
significant and non-significant features in $\thetab_m^{\true}$, respectively. 
Note that the separation between $\U_m$ and $\V_m$ is \textit{not known} 
during training and is simply used for mathematical convenience.

The goal of embedded VFL feature selection is to find a 
model that simultaneously gives similar predictions to the generating model $\bm\Theta^{\true}$ and sets non-significant feature weights to zero:
\begin{align}
    \min_{\bm\Theta} R(\bm\Theta ; \mathcal{P}_{\X,\y}) 
    \text{ s.t. } &\U_m^k \neq \zero~~\forall m \in [M],~\forall k \in [d_m^s] \nonumber \\
                  &\V_m^l = \zero~~\forall m \in [M],~\forall l \in [d_m^z]
    \label{fs.prob}
\end{align}
where $R(\cdot)$ is some generalization risk over the data distribution
$\mathcal{P}_{\X,\y}$ (e.g., expected squared loss, cross-entropy) 
and $d_m^s$ and $d_m^z$ are the number of significant and non-significant features at party $m$, respectively. 
Setting the input weights on non-significant features to zero removes their
influence in the network, thus it essentially removes the features from the model (shown visually in Figure~\ref{model.fig}).
Note $\V_m$ may not necessarily be zero in 
$\bm\Theta^{\true}$, as the effect of non-significant features can be
lost at any layer in the model $f(\bm\Theta^{\true})$.

A popular centralized method to solve \eqref{fs.prob} for neural networks is group lasso~\citep{ZhaoHW15, ZhangWSZP20, 0010ZWPP21}.
If we apply group lasso directly to the VFL setting, 
then we can define the estimator as follows:
\begin{align}
    \bar{\bm\Theta} \coloneqq \argmin_{\bm\Theta} R_N(\bm\Theta ; \X ; \y) 
    + \textstyle{\sum_{m=1}^{M}} \lambda_m G(\thetab_m) 
    \label{lasso.eq}
\end{align}
where $R_N(\cdot)$ is an empirical risk that approximates $R(\cdot)$ over $N$ training samples,
and $G(\cdot)$ is $L_{2,1}$ regularization:
\begin{align*}
    G(\thetab_m) \coloneqq
    \textstyle{\sum_{j=1}^{d_m}} \|\pi(\thetab_m)^j\|_2 
\end{align*}
where the projection $\pi(\cdot)$ extracts the input layer weights of 
$\thetab_m$ and $d_m$ is the number of input features.
The regularizer $G(\cdot)$ sparsifies the input layer weights on each 
feature, pushing irrelevant feature weights to zero.

\textbf{Why not standard group lasso?}
Minimizing the group lasso objective \eqref{lasso.eq} using 
standard VFL training~\citep{FDML, SplitNN}
requires the parties and server to exchange embeddings 
and partial derivatives every iteration of training (see Algorithm~\ref{vfl.alg}).
Instead of communicating embeddings at every iteration, is it possible
to perform feature selection locally at each client given auxiliary 
information from the server?
In the next section, we propose a feature selection method to 
solve~\eqref{fs.prob} that utilizes local training with minimal communication between the parties and the server.

\section{Algorithm} \label{alg.sec}

We present LESS-VFL, a communication-efficient approach 
to perform feature selection in VFL.
We formalize the three stages of LESS-VFL and present
a practical implementation.

\begin{algorithm}[t]
    \begin{algorithmic}[1]
    \STATE \textbf{Input}: pre-trained model parameters $\hat{\thetab}_s,\hat{\thetab}_1,\ldots,\hat{\thetab}_M$
    
    \FOR {$m \leftarrow 1, \ldots, M$ in parallel}
        \STATE Send $\h_m(\hat{\thetab}_m ; \X_m)$ to server
    \ENDFOR
    
    \STATE \textbf{Initialize}: $\thetab_s^0 \leftarrow \hat{\thetab}_s$
    \FOR {$t \leftarrow 0, \ldots, T_1-1$}
        \STATE Randomly sample $\B \subset [N]$
        \STATE $\hat{\Phi} \leftarrow 
            \{\thetab_s^t, \h_1(\hat{\thetab}_1 ; \X^{(\B)}_m), 
            \ldots, \h_M(\hat{\thetab}_M ; \X^{(\B)}_m)\}$
        \STATE $\thetab_s^{t+1} \leftarrow \text{prox}_{\lambda_s, \eta_s^t}\left(\thetab_s^t - \eta_s^t
                \nabla_s R_{\B}(\hat{\Phi} ; \y^{(\B)})\right)$  
    \ENDFOR

    \FOR {$m \leftarrow 1, \ldots, M$ in parallel}
        \STATE $\mathcal{K}_m = \{k ~|~ \| \pi(\thetab_s^{T_1})^k \|_2 > 0 \}$       
        \STATE \textbf{Initialize}: $\thetab_m^0 \leftarrow \hat{\thetab}_m$
    \ENDFOR

    \FOR {$t \leftarrow 0, \ldots, T_2-1$}
        \FOR {$m \leftarrow 1, \ldots, M$ in parallel}
            \STATE Randomly sample $\B \subset [N]$
            \STATE $\thetab_m^{t+1} = \text{prox}_{\lambda_m, \eta_m^t}\left(\thetab_m^t - \eta_m^t \nabla H_{\B}(\thetab_m^t ; \hat{\thetab}_m; \mathcal{K}_m) \right)$
        \ENDFOR
    \ENDFOR
    \STATE $\bar{\bm\Theta} \leftarrow [\thetab_s^{T_1}, \thetab_1^{T_2}, \ldots, \thetab_M^{T_2}]$
    \end{algorithmic}
    \caption{LESS-VFL implemented using P-SGD}
    \label{fvfl.alg}
\end{algorithm}

\textbf{Stage 1 -- Pre-training.}
The parties and server begin by solving the following empirical risk minimization problem:
\begin{align}
    \hat{\bm\Theta} \coloneqq \textstyle{\argmin_{\bm\Theta}}~ R_N(\bm\Theta ; \X ; \y).
    \label{pretrain.eq}
\end{align}
Standard VFL training, described in Algorithm~\ref{vfl.alg} in 
Appendix~\ref{algapp.sec}~\citep{FDML, SplitNN}, 
is a practical method to find an approximate solution to \eqref{pretrain.eq}. %

\textbf{Stage 2 -- Embedding Component Selection.}
In this stage, the server determines the set of significant
embedding components.
Each party sends its current pre-trained embeddings
$\h_m(\hat{\thetab}_m ; \x^{(i)})$ for each sample $i$.
These embeddings are fixed and used as input to the server model during this stage.
With some abuse of notation, we let 
$R_N(\thetab_s, \h_1(\hat{\thetab}_1), \ldots, \h_M(\hat{\thetab}_M))$
be the empirical risk of the server model using pre-trained embeddings
as inputs. The server solves the following:
\begin{align}
    \bar{\thetab}_s \coloneqq \argmin_{\thetab_s} 
        R_N(\thetab_s, \h_1(\hat{\thetab}_1), \ldots, \h_M(\hat{\thetab}_M))
        + \lambda_s G(\thetab_s)
    \label{server_lasso.eq}
\end{align}
where $G(\cdot)$ is the $L_{2,1}$ regularizer on the input layer of
$\thetab_s$
and $\hat{\thetab}_m$ are party $m$'s pre-trained parameters after pre-training. 
Note that the server uses the pre-trained embeddings as input, and does not
require communication with the parties to calculate $R_N(\cdot)$.
Solving \eqref{server_lasso.eq} simultaneously minimizes the risk while
sparsifying embedding component weights.
This is illustrated in Figure~\ref{model.fig}, where
non-significant embedding components (in gray) no 
longer provide input to the server model.

In Algorithm~\ref{fvfl.alg} (lines 1--10) we provide a practical method to find an approximate solution to \eqref{server_lasso.eq}.
The parties generate embeddings for all data samples using the pre-trained models and send them to the server (lines 1--4).
The server then starts embedding component selection (lines 5--10).
The server randomly samples a mini-batch $\B$ of indices,
then calculates the partial derivative of the risk with respect to 
the server model: $\nabla_s R_{\B}(\cdot)$.
The server then employs proximal stochastic gradient descent (P-SGD).
We let $\text{prox}_{\lambda, \eta}(\thetab)$
with parameter $\lambda$ and step size $\eta$
denote the closed-form solution to the proximal operator for $L_{2,1}$ regularization:
\begin{align*}
    \text{prox}_{\lambda,\eta}(\Pm^j) = 
    \begin{cases} 
        \Pm^j - \lambda \eta \frac{\Pm^j}{\|\Pm^j\|_2} & \| \Pm^j \|_2 > \lambda \eta \\
        \zero & \| \Pm^j \|_2 \leq \lambda \eta
    \end{cases}
\end{align*}
where $\Pm^j \coloneqq \pi(\thetab)^j$ is the $j$-th group of input weights.
After training, any embedding components with non-zero input weights are 
considered significant, and each party $m$ is sent its list of  
significant embedding components indices $\mathcal{K}_m$ (lines 11--14).

\textbf{Stage 3 -- Feature Selection.}
In this stage, each party's goal is to match the values of the significant 
embedding components while removing non-significant features from its model.
We denote the squared difference between the party's embedding value and the pre-trained embedding values over the set of significant components $\mathcal{K}_m$:
\begin{align*}
    &e(\thetab_m ; \hat{\thetab}_m ; \mathcal{K}_m ; \x_m^{(i)}) \coloneqq 
    \\ &~~~~~~~~~~~~~~~~~~~~~~~~
    \textstyle{\sum_{k \in \mathcal{K}_m}}
    (\h_m(\thetab_m ; \x_m^{(i)})^k - \h_m(\hat{\thetab}_m ; \x_m^{(i)})^k)^2
\end{align*}
where $\h_m(\thetab_m ; \x_m^{(i)})^k$ is the $k$-th embedding component.
Each party minimizes $e(\cdot)$ for each data sample
while sparsifying its input layer weights:
\begin{align}
    \bar{\thetab}_m \coloneqq \argmin_{\thetab_m} H_N(\thetab_m ; \hat{\thetab}_m ; \mathcal{K}_m ; \X_m) 
    + \lambda_m G(\thetab_m)
    \label{client_lasso.eq}     
\end{align}
where,
\begin{align*}
    H_N(\thetab_m ; \hat{\thetab}_m ; \mathcal{K}_m ; \X_m) \coloneqq 
    \frac{1}{N} \textstyle{\sum_{i=1}^N} e(\thetab_m ; \hat{\thetab}_m ; \mathcal{K}_m ; \x_m^{(i)}).
\end{align*}
A practical method to find an approximate solution to \eqref{client_lasso.eq}
can be seen in Algorithm~\ref{fvfl.alg} (lines 15--20).
Each party updates its model using
the mini-batch gradient $\nabla H_{\B}(\cdot)$ and applying
$\text{prox}_{\lambda_m, \eta_m^t}(\thetab)$ with
regularization parameter $\lambda_m$.

After minimizing \eqref{client_lasso.eq}, any input feature weights
set to zero are considered non-significant and removed from the model.
This is illustrated in Figure~\ref{model.fig}, where input weights from non-significant features (in gray) are removed.
Once feature selection is complete, 
one can further refine the network with the 
remaining features using Algorithm~\ref{vfl.alg} if desired. 

\textbf{Algorithm Cost.}
Stage 1 of LESS-VFL is the same as standard VFL, and thus has the same 
communication cost per iteration.
Stages 2 and 3 only require one round of communication where
parties send current embeddings to the server. 
The number of iterations $T_1$ and $T_2$ in Algorithm~\ref{fvfl.alg}
controls the computation cost at the server and parties respectively, which one can tune.

\textbf{Privacy.} 
LESS-VFL uses information already shared during VFL training to perform feature selection. Thus it provides the same privacy guarantees as standard VFL.
Although no raw data is shared between parties,
VFL may be vulnerable to reconstruction attacks and label leakage.
There have been techniques applied on top of VFL to protect against such attacks~\citep{Qiu2022hashing, zou2022defending},
and these can be similarly applied to LESS-VFL.
Our analysis in Section~\ref{analysis.sec} still holds
when applying these methods.

\textbf{Theory vs. Practice.}
We note that Algorithms~\ref{fvfl.alg} and~\ref{vfl.alg} provide approximate
solutions to each stage's optimization problem when running a fixed number of iterations.
By using P-SGD, input weights are set to zero and features are selected 
without the need for convergence, even if it is not the optimal set.
Our analysis of LESS-VFL in Section~\ref{analysis.sec} considers the 
ideal case where \eqref{pretrain.eq}, \eqref{server_lasso.eq}, and \eqref{client_lasso.eq} are solved exactly.
However, we find in our experiments in Section~\ref{exp.sec} that LESS-VFL
can remove spurious features and achieve high accuracy with only a 
few communication rounds for pre-training and an approximate solution 
from Algorithm~\ref{fvfl.alg}.

\rev{\textbf{Hyper-parameter tuning.}
Determining the best hyper-parameters for LESS-VFL (e.g. regularization parameters, number of pre-training epochs) can be done in an efficient manner. The parties and server can produce several pre-trained models for Stage 1 with different numbers of iterations. In Stage 2, the server can then explore the space of server model regularization parameters $\lambda_s$ without communication. For Stage 3, the server can share the resulting sets of significant embedding component indices with the parties, and each party then can tune its local regularization parameter $\lambda_m$ without communication.
For choosing the numbers of iterations $T_1$ in Stage 2, the server can minimize its optimization problem locally until the training loss plateaus. Similarly, for choosing $T_2$ in Stage 3, each party can minimize its local feature selection problem until its proxy training loss plateaus.}

\section{Theoretical Analysis} \label{analysis.sec}

We analyze LESS-VFL and prove under which conditions the 
algorithm minimizes risk and removes spurious features.
We assume each party $m$'s network is structured as follows:
\begin{itemize}[leftmargin=*, nosep]
    \item input layer: $\lv^1_m(\x_m) = \Pm_m \cdot \x_m + \p_m$
    \item hidden layers: 
    
    ~~~~~$\lv^j_m(\x_m)= \zeta^j_m(\Sm_m^j, \lv_m^{j-1}(\x_m),\ldots, \lv_m^{1}(\x_m))$
    \item output layer: $\h_m(\thetab_m ; \x_m) = \Q_m \cdot \lv_m^{L-1}(\x_m) + \q_m$
\end{itemize}
where $d^j_m$ are the number of neurons in the $j$-th hidden layer for party $m$,
$\Pm_m \in \mathbb{R}^{d^1_m \times d_m}$, $\Q_m \in \mathbb{R}^{d^{L}_m \times d^{L-1}_m}$, $\p_m \in \mathbb{R}^{d^1_m}$, $\q_m \in \mathbb{R}^{d^L_m}$,
and $\zeta^j_m(\cdot)$ are functions of previous layers with parameters $\Sm_m^j$.
We define the server network structure the same, denoted with subscript $s$.
This structure generalizes to several types of neural networks,
including feed-forward networks, convolutional networks, and 
many residual networks~\citep{DinhH20Neurips}.

We make the following assumptions, standard in model-based feature 
selection~\cite{huang2010variable, wu2009variable, DinhH20Neurips}:
\begin{assumption} \label{data.assum}
    Training data $\{(\x^{(i)}, y^{(i)})\}_{i=1}^N$ are sampled i.i.d. from distribution $\mathcal{P}_{\X, \y}$
    such that the input density $p_{\X}$ is positive and continuous on its 
    open domain $\mathcal{X}$
    and ${y^{(i)} = f(\bm\Theta^{\true} ; \x^{(i)}) + \epsilon^{(i)}}$ where $\epsilon \sim \mathcal{N}(0, \sigma^2)$. 
\end{assumption}
Assumption~\ref{data.assum} states that there is a
generating model $f(\bm\Theta^{\true})$ that generates the training labels with some Gaussian noise.
This ensures that feature selection is possible since the learned model
$f(\bar{\bm\Theta})$ matches the structure of the generating model.
The assumption on the input density $p_{\X}$ ensures that there are 
no perfect correlations between input features.
Note that since $\mathcal{X}$ is an open domain, we assume that 
all underlying features are continuous for this analysis.
\begin{assumption} \label{analytic.assum}
    The hidden layer functions $\zeta_m^j(\cdot)$ in all models are analytic. 
    The empirical risk is mean squared error:
$
    R_N(\bm\Theta ; \X ; \y) \coloneqq
    \frac{1}{N} \sum_{i=1}^{N} (f(\bm\Theta ; \x^{(i)}) - y^{(i)} )^2
$
and the generalization risk function is expected squared error:
$
    R(\bm\Theta ; \mathcal{P}_{\X,\y}) \coloneqq 
    \mathbb{E}_{(\x,y) \sim \mathcal{P}_{\X,\y}} [ \left(f(\bm\Theta ; \x) - y \right)^2 ].
$
\end{assumption}
Assumption~\ref{analytic.assum} ensures that 
the risk function is analytic, which allows us to reason about the distance between the learned model and the generating model.
Note that under the definition of the generating model in Assumption~\ref{data.assum}, $\bm\Theta^{\true}$ minimizes the expected squared error $R(\cdot)$.

\rev{Next, we formalize our goal to find parameters 
$\bar{\bm\Theta}$ that solves \eqref{fs.prob}, i.e. performs the same as the generating model $\bm\Theta^{\true}$
while removing non-significant features.}
We define the set $\mathcal{T}^*$ as the parameters that achieve the same risk as the generating model:
$$\mathcal{T}^* \coloneqq \{\bm\Theta : R(\bm\Theta) = R(\bm\Theta^{\true})\}.$$
\rev{Recall from Section~\ref{problem.sec} that for $\bm\Theta^{\true}$, it is not necessarily the case that the input weights on non-significant features are zero. 
The same holds for any model in $\mathcal{T}^*$.
Thus, we define a subset of parameters in $\mathcal{T}^*$} 
that also have weights on non-significant features set to zero:
$$\mathcal{T}_{\phi}^* \coloneqq  \{\bm\Theta : \bm\Theta \in \mathcal{T}^* \text{ and } \V_m = \zero \},$$
where $\V_m$ are the input weights on non-significant features in the 
generating model $\bm\Theta^{\true}$ (definition in Section~\ref{problem.sec}).
We define the distance of a vector $\thetab$ from a set of vectors $\mathcal{S}$ as:
$$d(\thetab, \mathcal{S}) = \inf_{\thetab' \in \mathcal{S}} \| \thetab - \thetab' \|_2.$$
\rev{The feature selection problem~\eqref{fs.prob} is solved if $d(\bar{\bm\Theta}, \mathcal{T}_{\phi}^*) \rightarrow 0$ for our learned parameters $\bar{\bm\Theta}$.}
\citet{DinhH20Neurips} proved this can be achieved using group lasso in centralized machine learning problems.
Our goal is to show that our three-stage method can achieve the same in VFL settings.

\subsection{Main Result}

We present our main result below.
\begin{theorem} \label{main_feature.thm}
Let $\tilde{\bm\Theta} = \argmin_{\bm\Theta \in \mathcal{T}^*} \| \bm\Theta - \hat{\bm\Theta} \|_2$. 
    Let $\mathcal{K}_m$ and $\mathcal{Z}_m$ be the set of party $m$'s significant and non-significant embedding components in the model $\tilde{\bm\Theta}$, respectively.
    Let $\bar{\bm\Theta}$ be the model after solving \eqref{pretrain.eq}, 
\eqref{server_lasso.eq}, and \eqref{client_lasso.eq} in succession.
Under Assumptions \ref{data.assum} and \ref{analytic.assum}, we conclude:

    (i) After solving \eqref{server_lasso.eq} and obtaining $\bar{\thetab}_s$, for all $m$,
    $\pi(\bar{\thetab}_s)^k \neq \zero$ for all $k \in \mathcal{K}_m$ and
    $\pi(\bar{\thetab}_s)^l \rightarrow \zero$ for all $l \in \mathcal{Z}_m$,
    where $\pi(\bar{\thetab}_s)$ are the embedding layer weights.
    
(ii) 
When the server finds $\mathcal{K}_m$ for all $m$, then
for any $\delta > 0$ and $\nu > 1$, 
if $\lambda_s \sim N^{-1/4}$ and
$\lambda_m \sim N^{-1/4}$ for all $m$,
    \begin{align}
    d(\bar{\bm\Theta}, \mathcal{T}_{\phi}^*) = 
    O\left( \sqrt{M} \left( \frac{\log N}{N} \right)^{\frac{1}{4(\nu-1)}} \right)
    \label{main_feature.eq}
    \end{align}
    with probability $1-\delta$.
\end{theorem}

Theorem~\ref{main_feature.thm} (i) states that the server finds a set 
of embedding components that are significant in $\tilde{\bm\Theta}$, the closest model to the pre-trained model $\hat{\bm\Theta}$
that matches the risk of the generating model.
This result ensures that the list of embedding components given by
the server to the parties can serve as an accurate proxy for the loss function.

Theorem~\ref{main_feature.thm} (ii) states if we run each LESS-VFL stage
to convergence, then we approach a model that minimizes the risk and
removes non-significant features at a polynomial rate 
in terms of the number of training samples $N$. 
Since parties cannot calculate the risk $R_N(\cdot)$ locally, each uses
$H_N(\cdot)$, the distance between the produced embeddings and the significant components of the pre-trained embeddings, as a proxy. 
It is not immediately evident that feature selection can be performed at 
each party without access to the server model to calculate the risk.
Theorem~\ref{main_feature.thm} states that regardless of the depth or complexity of the server model,
given pre-trained embeddings from solving \eqref{pretrain.eq} 
and the set of significant embedding components $\mathcal{K}_m$ found by solving \eqref{server_lasso.eq},
each party can successfully remove its non-significant features by solving \eqref{client_lasso.eq}.
This emphasizes that all stages of LESS-VFL
are necessary.%

The bound in \eqref{main_feature.eq} is similar to that of centralized group 
lasso~\cite{DinhH20Neurips}, with the addition of sub-linear
error growth depending on the number of parties $M$.
It is common for $M$ to be small in many VFL applications~\cite{kairouz2019advances}, 
thus this term has a minor effect on the bound.

\subsection{Proof Sketch}
For the sake of brevity, we present this proof sketch for the case 
where $M=1$ (one party and server), using subscript $m$ to denote the 
party. We provide the complete proof of Theorem~\ref{main_feature.thm} for $M>1$ in the appendix.
The proof for $M>1$ is similar to that of $M=1$ since the server-side
group lasso treats embeddings as input and is agnostic to the number of parties,
and party-side group lasso runs in parallel using only its own significant
embedding components as a proxy for the loss function.
The key challenge in the extension to $M>1$ comes in the relationship
between significant embedding components and significant party features (see Lemma~\ref{sig2.lemma} in Appendix~\ref{extension.sec}).

We start by providing some definitions and additional notation.
We define $H(\cdot)$ as the expected squared difference between two 
embeddings over the full data distribution $\mathcal{P}_{\X,\y}$:
\begin{align*}
&H(\thetab_m ; \thetab_m' ; \mathcal{K}_m) \coloneqq 
\mathbb{E}_{(\x,y) \sim \mathcal{P}_{\X,\y}} \left[ 
e(\thetab_m ; \thetab_m'; \mathcal{K}_m ; \x_m)
\right]
\end{align*}
where $\thetab_m, \thetab_m'$ are party model parameters,
$\mathcal{K}_m$ is a set of embedding components, and $e(\cdot)$ is the 
square difference between embeddings components in $\mathcal{K}_m$.
Recall that the notion of significance as given in 
Definition~\ref{sig.def} can be applied to any input and model.
We summarize our steps to prove that $d(\bm\Theta, \mathcal{T}_{\phi}^*) \rightarrow 0$:
\begin{enumerate}[leftmargin=*, nosep, label=(\alph*)]
    \item Prove that minimizing $H(\cdot)$ also minimizes $R(\cdot)$ if $\mathcal{K}_m$ is the set of significant embedding components.
    \item Prove that $e(\cdot~;~\mathcal{K}_m)$ has the same significant and non-significant features as $f(\cdot)$ if $\mathcal{K}_m$ is the set of significant embedding components.
    \item Prove that \eqref{server_lasso.eq} finds optimal server parameters and finds the set of significant embedding components.
    \item Prove that given the set of significant embedding components, \eqref{client_lasso.eq} finds optimal party parameters and removes non-significant features.
\end{enumerate}

We start by proving (a) and (b).
In the following proposition, we discuss the relationship between the significance of features in the full network versus the significance of
features to embedding components in the party sub-network.
\begin{proposition} \label{sig.prop}
Consider a model $\bm\Theta = [\thetab_s^{\top}, \thetab_m^{\top}]^{\top}$.
Let $\s$ and $\z$ be the sets of significant and non-significant features for $f(\bm\Theta)$, respectively.
Let the set of significant embedding components for $f(\thetab_s ; \h(\thetab_m; \x))$ be $\s_s$.
Let $g^j(\x, r)$ replace input $\x^j$ with value $r$.
For each significant embedding component $k \in \s_s$, for all $j \in \z$, and any $r \in \mathbb{R}$, $\h(\thetab_m; \x)^k = \h(\thetab_m; g^j(\x, r))^k$.
\end{proposition}
Informally, Proposition~\ref{sig.prop} states that significant 
embedding component values are unchanged by non-significant 
features, and can \emph{only} be changed by significant features.
\begin{proof}
Suppose that
$\h(\thetab_m; \x)^k \neq \h(\thetab_m; g^j(\x, r))^k$
for some significant embedding component $k \in \s_s$, non-significant feature $j \in \z$, and $r \in \mathbb{R}$.
By our supposition and
since component $k$ is significant, 
$f(\thetab_s ; \h(\thetab_m ; \x )) \neq f(\thetab_s ; \h(\thetab_m ; g^j(\x ; r)))$ for some value $r \in \mathbb{R}$. 
This contradicts the fact that $j$ is a non-significant feature.
\end{proof}

Utilizing Proposition~\ref{sig.prop}, we can prove (a) and (b).
\begin{lemma} \label{sig.lemma}
Let $\tilde{\bm\Theta} = [\tilde{\thetab}_s^{\top}, \tilde{\thetab}_m^{\top}]^{\top} \in \mathcal{T}^*$.
Let $\mathcal{K}_m$ be the significant embedding components for $f(\tilde{\thetab}_s ; \h(\tilde{\thetab}_m))$.
Let $\thetab_m = \argmin_{\thetab'_m} H(\thetab'_m ; \tilde{\thetab}_m ; \mathcal{K}_m)$.
Let $\s$ and $\z$ be the significant and non-significant features for $f(\bm\Theta^{\true})$.
Let $\s_h$ and $\z_h$ be the significant and non-significant features for $e(\thetab_m ; \tilde{\thetab}_m ; \mathcal{K}_m)$ with parameters $\thetab_m$.
Then: $$[\tilde{\thetab}_s, \thetab_m] \in \mathcal{T}^*,~
\s_h = \s, \text{ and } \z_h = \z.$$
\end{lemma}

The proof of Lemma~\ref{sig.lemma} is given in Appendix~\ref{lemmasig.sec}.
Lemma~\ref{sig.lemma} proves that $H(\cdot)$ can
serve as a proxy for $R(\cdot)$ if the pre-trained model parameters $\hat{\bm\Theta}$ are in the optimal set $\mathcal{T}^*$
and the set of selected embedding components $\mathcal{K}_m$ contains only 
the significant embedding components. 

Next, we consider objective (c), and prove that the server can find
the significant embedding components required for (a) and (b).
We first define $\tilde{\bm\Theta} = \argmin_{\bm\Theta \in 
\mathcal{T}^*} \| \bm\Theta - \hat{\bm\Theta} \|_2$ as the closest
model in $\mathcal{T}^*$ to the pre-trained parameters $\hat{\bm\Theta}$.
The server's goal is to find
the set of significant embedding components in $\tilde{\bm\Theta}$,
which the parties can then use to remove their non-significant features.
We define the set of server model parameters in $\mathcal{T}^*$ with 
non-significant embedding component weights set to zero:
$$\mathcal{S}_{\phi}^* = \{\thetab_s : \exists \thetab_m \text{ s.t. } 
[\thetab_s, \thetab_m] \in \mathcal{T}^* \text{ and } \V_s = \zero \}$$
where $\V_s$ are the weights on non-significant embedding components in 
$\tilde{\bm\Theta}$.
If $d(\bar{\thetab}_s, \mathcal{S}_{\phi}^*) \rightarrow 0$, the server finds
the set of significant embedding components, completing objective (c).
We bound this distance in the following theorem.
\begin{theorem} \label{server.thm}
Given a pre-trained model $\hat{\bm\Theta}$ defined by \eqref{pretrain.eq},
for any $\delta > 0$, there exists $N \geq N_0(\delta)$ such that:
    \begin{align*}
        d(\bar{\thetab}_s, \mathcal{S}_{\phi}^*)
        = O \left( \frac{\log N}{\lambda_s\sqrt{N}}
    + \left(\frac{\log N}{\sqrt{N}} + \lambda_s^{\nu / (\nu - 1)} \right)^{1/\nu} \right)
    \end{align*}
    with probability $1-\delta$.
\end{theorem}

The proof of Theorem~\ref{server.thm} can be found in Appendix~\ref{serverthm.sec}. The bound in Theorem~\ref{server.thm} indicates 
that if $\lambda_s \sim N^{-1/4}$, then non-significant weights will 
approach zero at a polynomial rate in terms of the number of training samples $N$;
if the regularization parameter is set appropriately, 
then the set of significant embedding components are found.

Finally, since the server finds the set of significant embedding components allowing $H(\cdot)$ to be a proxy of $R(\cdot)$, 
we can prove objective (d).
We define the set of party models in $\mathcal{T}^*$ 
with non-significant feature weights set to zero:
$$\mathcal{C}_{\phi}^* = \{\thetab_m : \exists \thetab_s \text{ s.t. } 
[\thetab_s, \thetab_m] \in \mathcal{T}^* \text{ and } \V_m = \zero \}$$
where $\V_m$ are the input weights on non-significant features in the 
generating model $\bm\Theta^{\true}$.
If $d(\bar{\thetab}_m, \mathcal{C}_{\phi}^*) \rightarrow 0$, then the party removes the non-significant features, completing objective (d).
We bound this distance in the following theorem.
\begin{theorem} \label{client.thm}
Let $[\tilde{\thetab}_s^{\top}, \tilde{\thetab}_m^{\top}]^{\top} = \argmin_{\bm\Theta \in \mathcal{T}^*} \| \bm\Theta - \hat{\bm\Theta} \|_2$
where $\hat{\bm\Theta}$ are pre-trained model parameters defined in \eqref{pretrain.eq}.
If $\mathcal{K}_m$ in \eqref{client_lasso.eq} is the set of
significant embedding components for $f(\tilde{\thetab}_s ; \h(\tilde{\thetab}_m))$,
then for any $\delta > 0$, there exists some number of samples $N \geq N_0(\delta)$ such that:
    \begin{align*}
        &d(\bar{\thetab}_m, \mathcal{C}_{\phi}^*)
        \leq O\left( \frac{\log N}{\lambda_m\sqrt{N}} 
        + \left(\frac{\log N}{\sqrt{N}} + \lambda_m^{\nu / (\nu-1)}\right)^{1/\nu} \right)
    \end{align*}
    with probability $1-\delta$.
\end{theorem}
The proof of Theorem~\ref{client.thm} can be found in Appendix~\ref{clientthm.sec}. Similar to the server's case, this bound goes to zero at a polynomial rate if $\lambda_m \sim N^{-1/4}$.
Note that for Theorem~\ref{client.thm} to hold, the server must provide 
the party with the set of significant embeddings.
Otherwise, the bound is not guaranteed.
In Section~\ref{exp.sec}, we explore the importance of the embedding
component selection stage in practice. 
In Appendix~\ref{extension.sec}, we extend
Theorems~\ref{server.thm} and~\ref{client.thm} to the case where $M>1$,
which can be combined to prove Theorem~\ref{main_feature.thm}.

\section{Experiments} \label{exp.sec}
We implement LESS-VFL by running a fixed number of iterations of Algorithm~\ref{vfl.alg} (standard VFL algorithm, see Appendix~\ref{algapp.sec}), then running 
Algorithm~\ref{fvfl.alg} to remove non-significant features, 
then continuing training with Algorithm~\ref{vfl.alg}.
\rev{We evaluate LESS-VFL on several datasets.
\begin{itemize}[leftmargin=*, nosep]
    \item \textbf{MIMIC-III} \cite{johnson2016mimic,Harutyunyan2019}: Hospital dataset consisting of time-series medical information on anonymized patients. Used to predict in-hospital mortality. Contains $14$,$681$ samples each with $712$ features.    
    \item \textbf{Activity} \cite{AnguitaGOPR13}: Time-series positional data on humans performing various activities. Used for multi-class classification of the current activity (walking, sitting, running, etc.).  Contains $7$,$352$ samples each with $560$ features.
    \item \textbf{Phishing} \cite{Dua:2019}: Dataset that provides relevant features for determining if a website is a phishing website (use of HTTP, TinyURL, forwarding, etc.). Contains $11$,$055$ samples each with $30$ features.
    \item \textbf{Gina} \cite{agnostic}: Hand-written two-digit images. Used for binary classification between even and odd numbers, meaning only the first digit is necessary for classification and the rest of the features are distractions. Contains $3$,$468$ samples each with $968$ features.
    \item \textbf{Sylva} \cite{agnostic}: Forest cover type information. Used for binary classification (Ponderosa pine vs. everything else). Similar to Gina, half the features are distractions; each sample has two records with relevant information for the target, while the other two are randomly chosen.  Contains $14$,$395$ samples each with $216$ features.
\end{itemize}}

For each dataset, we add $50\%$ more features that are
Gaussian noise. These spurious features act as our non-significant features, allowing us to measure how well LESS-VFL performs feature selection.
\rev{Note that not all the features in the original dataset are necessarily significant. The only features we know for sure are non-significant are the Gaussian noise features we add to each dataset. Thus, the final test accuracy is our indicator that we have correctly selected significant features in the dataset and trained a model that generalizes well.}

We compare LESS-VFL with the following VFL baselines.
\begin{itemize}[leftmargin=*, nosep]
    \item \textbf{VFL (Original)}: VFL as described in Algorithm~\ref{vfl.alg} \mbox{\textit{without}} spurious features in the datasets.
    \item \textbf{VFL (Spurious)}: Algorithm~\ref{vfl.alg} \textit{with} spurious features in the datasets.
    \item \textbf{Group Lasso}: Applies group lasso directly to the VFL model by approximately solving \eqref{lasso.eq} using P-SGD.
    \item \textbf{Local Lasso}: 
    This algorithm is the same as LESS-VFL with stage 2, embedding component selection, removed.
\end{itemize}
We restrict our evaluations to methods that do not require a fully trained 
VFL model as input, which excludes \citet{ZhangLHCZ22} and \citet{ChenDLWH22} from our comparison.
The feature selection portion of VFLFS~\cite{Feng22Group} employs 
group lasso, which we include in our evaluations.

\textbf{Training Details.}
\rev{For each dataset, we split both the original and Gaussian noise features evenly among
a set of parties} (three parties for Phishing, four parties otherwise). 
Each party's model is a 3-layer dense neural network,
and the server trains a linear model that takes the 
concatenation of party embeddings as input.
We run a grid search to determine regularization parameters for LESS-VFL, local lasso, and group lasso, and the number of pre-training epochs for LESS-VFL and local lasso.
We chose parameters that achieved the highest training accuracy and 
removed at least $80\%$ of spurious features.
We use the ADAM optimizer with a learning rate of $0.01$ 
when employing Algorithm~\ref{vfl.alg} in
VFL (Original and Spurious) and pre-training and post feature selection in local lasso and LESS-VFL.
We run $150$ epochs of P-SGD for embedding component selection in LESS-VFL and feature selection in LESS-VFL and local lasso,
which we found to be a sufficient amount of iterations for the 
training loss to plateau.

\begin{table}
    \caption{Communication cost to achieve $90\%$ of baseline test accuracy and remove at least $80\%$ of the spurious features. The value shown is the average of $5$ runs $\pm$ the standard deviation.}
\label{comm.table}
\vskip 0.1in
\small
\centering
\resizebox{0.49\textwidth}{!}{
\begin{tabular}{lccc}
    \toprule 
    \multirow{3}{*}{\textbf{Dataset}}& \multicolumn{3}{c}{\textbf{Communication Cost (MB)}}  \\ 
      \cmidrule(rl){2-4}  
      & \textbf{Group Lasso} & \textbf{Local Lasso} & \textbf{LESS-VFL (ours)} \\
\midrule

MIMIC-III & 57.35 $\pm$ 0.00 & 30.47 $\pm$ 1.79 & \textbf{7.17 $\pm$ 0.00} \\
Activity & 322.73 $\pm$ 61.32 & 26.56 $\pm$ 5.83 & \textbf{21.17 $\pm$ 0.00} \\

Phishing & 95.22 $\pm$ 1.89 & 8.10 $\pm$ 3.40 & \textbf{3.99 $\pm$ 0.75} \\
Gina & 13.55 $\pm$ 0.00 & 1.90 $\pm$ 0.27 & \textbf{1.48 $\pm$ 0.26} \\

Sylva & 22.49 $\pm$ 0.00 & \textbf{5.62 $\pm$ 0.00} & \textbf{5.62 $\pm$ 0.00} \\
    \bottomrule
\end{tabular}}
\end{table}

\textbf{Communication cost.}
In Table~\ref{comm.table}, we compare the communication cost of reaching
a target test accuracy while removing at least $80\%$ of the spurious features.
We choose a target accuracy of $90\%$ of the maximum accuracy reached by VFL (Original).
In all cases, LESS-VFL meets these 
conditions with the lowest communication cost, reducing the communication cost when compared to group lasso. 
In the case of the Phishing dataset, LESS-VFL has ${\sim}20 \times$ lower
communication cost than group lasso.
LESS-VFL greatly reduces the cost of feature selection over group lasso 
by only communicating during pre-training, and once at the start of feature selection.
LESS-VFL also always achieves the same or lower communication cost than local lasso.
Local lasso forgoes embedding component selection, and
in most datasets, this led to higher communication cost. We explore this more in our next set of experiments.

\begin{figure}[t]
    \begin{subfigure}{0.23\textwidth}
        \centering
        \includegraphics[width=\textwidth]{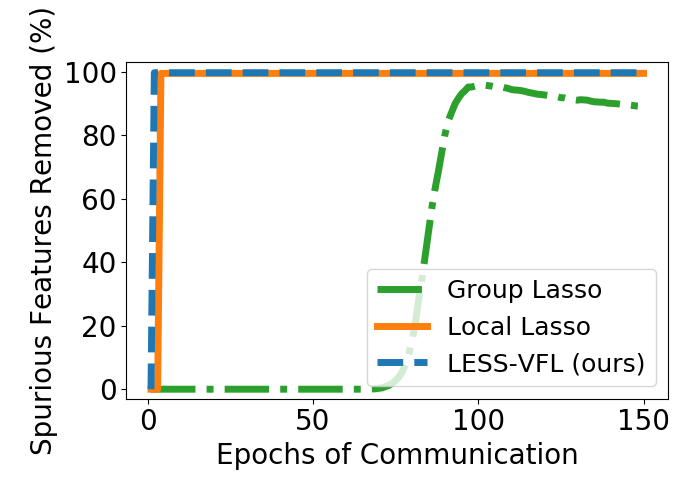}
        \vspace*{-7mm}
        \caption{Activity}
        \label{activitybar.fig}
    \end{subfigure}
    \hfill
    \begin{subfigure}{0.23\textwidth}
        \centering
        \includegraphics[width=\textwidth]{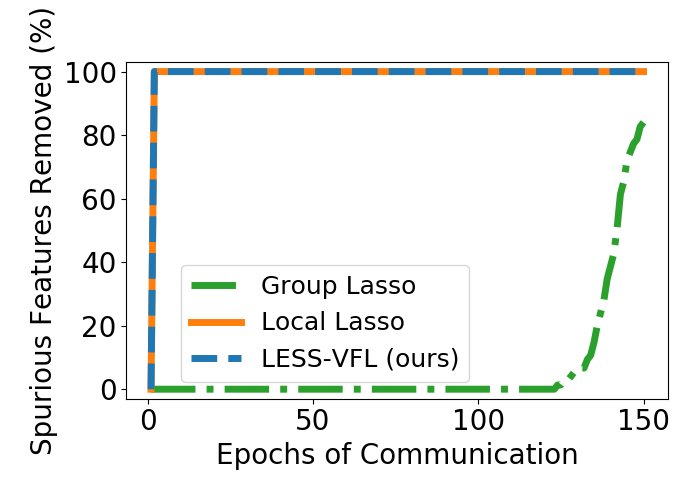}
        \vspace*{-7mm}
        \caption{Phishing}
        \label{phishingbar.fig}
    \end{subfigure}
    \hfill
    \caption{Communication rounds to remove spurious features. 
    The values shown is the average of $5$ runs.
    LESS-VFL and local lasso remove a similar percentage of spurious features after pre-training, though local lasso takes longer to reach high accuracy (see Table~\ref{comm_cost2.table}).
    Group Lasso gradually removes features while local lasso and LESS-VFL remove features with only a few rounds of communication after pre-training. 
    }
    \label{red.fig}
\end{figure}

In the remaining experiments, we seek to illustrate how LESS-VFL performs over the course of training.
We focus on two representative datasets (Activity and Phishing). 
We provide results for all datasets in Appendix~\ref{exp2.sec}.

\textbf{Feature removal.}
In Figure~\ref{red.fig}, we compare the percentage of spurious features removed using group lasso, local lasso, and LESS-VFL over the communication epochs.
Group lasso gradually removes features over the course of training, while 
local lasso and LESS-VFL remove features after a few rounds of communication for pre-training.
LESS-VFL benefits greatly from using local training without communication
to perform its feature selection.
Local lasso removes a similar percentage of features as LESS-VFL in about the same communication epochs.
However, we see in the next experiment that local lasso can require more 
communication to both reach high accuracy and remove spurious features.

\begin{figure}[t]
    \begin{subfigure}{0.49\textwidth}
        \centering
        \includegraphics[width=0.85\textwidth]{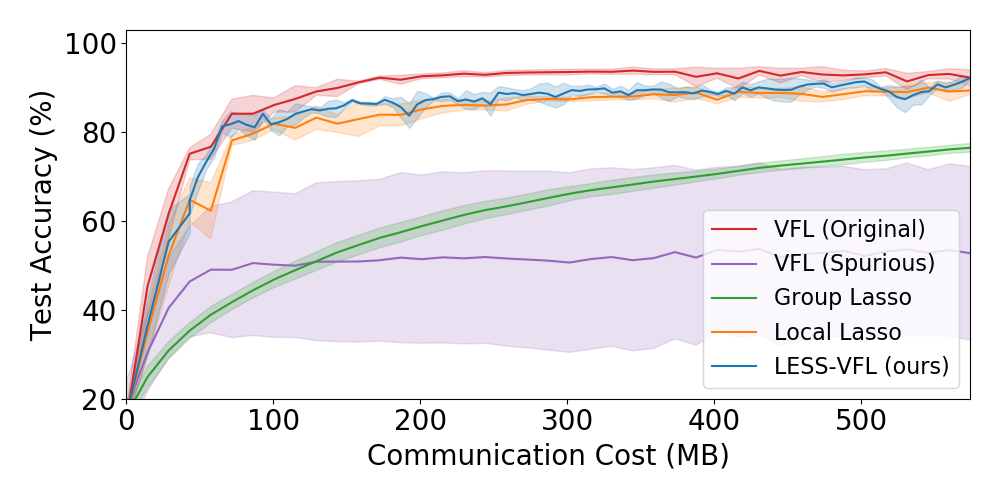}
        \vspace*{-3.5mm}
        \caption{Activity}
        \label{activityred5.fig}
    \end{subfigure}
    \hfill
    \begin{subfigure}{0.49\textwidth}
        \centering
        \includegraphics[width=0.85\textwidth]{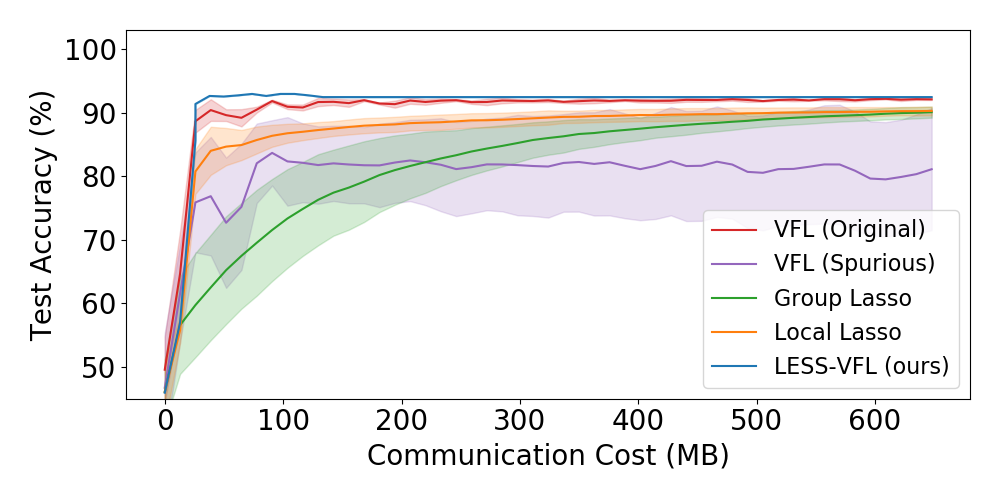}
        \vspace*{-3.5mm}
        \caption{Phishing}
        \label{phishingred5.fig}
    \end{subfigure}
    \hfill
    \vspace*{-5mm}
    \caption{Test accuracy plotted by communication cost. 
    VFL (Original) is trained \textit{without} spurious features, while all other methods are trained \textit{with} spurious features.
    The solid lines are the average of $5$ runs and the shaded region represents the standard deviation.}
    \label{acc.fig}
\end{figure}

\textbf{Accuracy.}
In Figure~\ref{acc.fig}, we plot the test accuracy against communication cost. 
The test accuracy of VFL (Spurious) in both datasets indicates
that VFL training without removing spurious features can have a drastic 
effect on the generalization.
In both cases, LESS-VFL achieves high accuracy faster than group lasso, 
and achieves similar accuracy to the baseline VFL algorithm without spurious features.
In fact, in the Phishing dataset, LESS-VFL performs better than the VFL 
(Original) baseline. This is due to LESS-VFL removing non-significant 
embedding components which reduce post feature selection communication cost.
Local lasso performs similarly to LESS-VFL in the Activity dataset, but 
local lasso requires a much higher communication cost to achieve high 
accuracy in the Phishing dataset.
In Figure~\ref{phishingred5.fig}, we can see local lasso has a lower model 
accuracy than LESS-VFL after feature selection (at ${\sim}25$ MB).
This reinforces that embedding component selection can improve model accuracy 
by both minimizing risk to refine server model parameters, and providing parties with important information for local feature selection.

\rev{\textbf{Uneven Features.}
For the previous experiments, we considered a case where all parties have the same percentage of Gaussian noise features.
We now consider a case where parties have an uneven distribution of Gaussian noise features: One party with $80\%$ additional Gaussian noise features, one with $25\%$, one with $10\%$, and one with no Gaussian noise features.
Table~\ref{comm_hetero.table} shows the communication cost of group lasso, local lasso, and LESS-VFL to reach $90\%$ of baseline VFL (Original) test accuracy while removing $80\%$ of the total spurious features.
We find that, in this heterogeneous setting, LESS-VFL still achieves high accuracy while removing spurious features, and does so with low communication cost. 

\begin{table}
    \caption{Experimental results with heterogeneous feature partitions. Communication cost to achieve $90\%$ of baseline test accuracy and remove at least $80\%$ of the spurious features. The value shown is the average of $5$ runs $\pm$ the standard deviation.}
\label{comm_hetero.table}
\vskip 0.1in
\small
\centering
\resizebox{0.49\textwidth}{!}{
\begin{tabular}{lccc}
    \toprule 
    \multirow{3}{*}{\textbf{Dataset}}& \multicolumn{3}{c}{\textbf{Communication Cost (MB)}}  \\ 
      \cmidrule(rl){2-4}  
      & \textbf{Group Lasso} & \textbf{Local Lasso} & \textbf{LESS-VFL (ours)} \\
\midrule
MIMIC-III & 57.35 $\pm$ 0.00 & \textbf{7.17 $\pm$ 0.00} & \textbf{7.17 $\pm$ 0.00} \\
Activity & 187.39 $\pm$ 52.61 & 24.77 $\pm$ 6.75 & \textbf{15.76 $\pm$ 3.09} \\

Phishing & 94.57 $\pm$ 1.94 & 5.51 $\pm$ 0.79 & \textbf{3.98 $\pm$ 0.74} \\
Gina & 13.55 $\pm$ 0.00 & 1.63 $\pm$ 0.33 & \textbf{1.35 $\pm$ 0.00} \\

Sylva & 21.93 $\pm$ 1.12 & \textbf{5.62 $\pm$ 0.00} & \textbf{5.62 $\pm$ 0.00} \\
    \bottomrule
\end{tabular}}
\end{table}}

\section{Conclusion} \label{conclusion.sec}
In this work, we proposed LESS-VFL, a communication-efficient method
for feature selection in vertical federated learning.
We analytically proved that LESS-VFL removes spurious features.
We experimentally showed that LESS-VFL can achieve comparable accuracy
and percentage of spurious features removed at reduced communication cost.
\rev{In the future, we seek to extend our analysis to non-analytic neural 
networks and adaptive group lasso.}

\section*{Acknowledgment}
This material is based upon work supported in part by the National Science Foundation under Grant CNS-1553340, and 
 by the Rensselaer-IBM AI Research Collaboration (\url{http://airc.rpi.edu}), part of the IBM AI Horizons Network.

\bibliography{references}

\begin{thebibliography}{33}
\providecommand{\natexlab}[1]{#1}
\providecommand{\url}[1]{\texttt{#1}}
\expandafter\ifx\csname urlstyle\endcsname\relax
  \providecommand{\doi}[1]{doi: #1}\else
  \providecommand{\doi}{doi: \begingroup \urlstyle{rm}\Url}\fi

\bibitem[Anguita et~al.(2013)Anguita, Ghio, Oneto, Parra, and
  Reyes{-}Ortiz]{AnguitaGOPR13}
Anguita, D., Ghio, A., Oneto, L., Parra, X., and Reyes{-}Ortiz, J.~L.
\newblock A public domain dataset for human activity recognition using
  smartphones.
\newblock In \emph{European Symp. Artif. Neural Net.}, 2013.

\bibitem[Canino et~al.(2016)Canino, Suo, Guzzi, Tradigo, Zhang, and
  Veltri]{CaninoSGTZV16}
Canino, G., Suo, Q., Guzzi, P.~H., Tradigo, G., Zhang, A., and Veltri, P.
\newblock Feature selection model for diagnosis, electronic medical records and
  geographical data correlation.
\newblock In \emph{Proc. {ACM} Int. Conf. Bioinfo. Comp. Bio. Health Info.},
  pp.\  616--621, 2016.

\bibitem[Castiglia et~al.(2022)Castiglia, Das, Wang, and
  Patterson]{castiglia2022compressed}
Castiglia, T., Das, A., Wang, S., and Patterson, S.
\newblock {Compressed-VFL}: Communication-efficient learning with vertically
  partitioned data.
\newblock In \emph{Proc. 39th Int. Conf. on Machine Learn.}, pp.\  2738--2766,
  2022.

\bibitem[Ceballos et~al.(2020)Ceballos, Sharma, Mugica, Singh, Roman,
  Vepakomma, and Raskar]{SplitNN}
Ceballos, I., Sharma, V., Mugica, E., Singh, A., Roman, A., Vepakomma, P., and
  Raskar, R.
\newblock Splitnn-driven vertical partitioning.
\newblock \emph{arXiv:2008.04137}, 2020.

\bibitem[Cha et~al.(2021)Cha, Sung, and Park]{verticalAutoencoders}
Cha, D., Sung, M., and Park, Y.-R.
\newblock Implementing vertical federated learning using autoencoders:
  Practical application, generalizability, and utility study.
\newblock \emph{JMIR Medical Informatics}, 9\penalty0 (6):\penalty0 e26598,
  2021.

\bibitem[Chandrashekar \& Sahin(2014)Chandrashekar and Sahin]{ChandrashekarS14}
Chandrashekar, G. and Sahin, F.
\newblock A survey on feature selection methods.
\newblock \emph{Comput. Electr. Eng.}, 40\penalty0 (1):\penalty0 16--28, 2014.

\bibitem[Chen et~al.(2022)Chen, Du, Lu, Wu, and Hung]{ChenDLWH22}
Chen, P., Du, X., Lu, Z., Wu, J., and Hung, P. C.~K.
\newblock {EVFL:} an explainable vertical federated learning for data-oriented
  artificial intelligence systems.
\newblock \emph{J. Syst. Archit.}, 126:\penalty0 102474, 2022.

\bibitem[Chen et~al.(2021)Chen, Zhou, Guan, Yang, Fao, Wang, and
  Wang]{ChenZGYFW021}
Chen, X., Zhou, S., Guan, B., Yang, K., Fao, H., Wang, H., and Wang, Y.
\newblock {Fed-EINI}: An efficient and interpretable inference framework for
  decision tree ensembles in vertical federated learning.
\newblock In \emph{{IEEE} Int. Conf. Big Data}, pp.\  1242--1248, 2021.

\bibitem[Clinciu \& Hastie(2019)Clinciu and Hastie]{clinciu2019survey}
Clinciu, M. and Hastie, H.
\newblock A survey of explainable ai terminology.
\newblock In \emph{Proc. Workshop Interactive Natural Lang. Tech. Explainable
  AI}, pp.\  8--13, 2019.

\bibitem[Dinh \& Ho(2020)Dinh and Ho]{DinhH20Neurips}
Dinh, V.~C. and Ho, L. S.~T.
\newblock Consistent feature selection for analytic deep neural networks.
\newblock In \emph{Adv. Neural Inf. Process. Syst.}, 2020.

\bibitem[Dua \& Graff(2017)Dua and Graff]{Dua:2019}
Dua, D. and Graff, C.
\newblock {UCI} machine learning repository, 2017.
\newblock URL \url{http://archive.ics.uci.edu/ml}.

\bibitem[Feng(2022)]{Feng22Group}
Feng, S.
\newblock Vertical federated learning-based feature selection with
  non-overlapping sample utilization.
\newblock \emph{Expert Syst. Appl.}, 208:\penalty0 118097, 2022.

\bibitem[Feng \& Yu(2020)Feng and Yu]{MultiVFL2020}
Feng, S. and Yu, H.
\newblock Multi-participant multi-class vertical federated learning.
\newblock \emph{arXiv}, abs/2001.11154, 2020.

\bibitem[Guyon(2007)]{agnostic}
Guyon, I.
\newblock Agnostic learning vs. prior knowledge.
\newblock IJCNN Workshop on Agnostic Learning vs. Prior Knowledge, 2007.
\newblock URL \url{http://www.agnostic.inf.ethz.ch/datasets.php}.
\newblock Accessed Apr. 27 2023.

\bibitem[Harutyunyan et~al.(2019)Harutyunyan, Khachatrian, Kale, Ver~Steeg, and
  Galstyan]{Harutyunyan2019}
Harutyunyan, H., Khachatrian, H., Kale, D.~C., Ver~Steeg, G., and Galstyan, A.
\newblock Multitask learning and benchmarking with clinical time series data.
\newblock \emph{Scientific Data}, 6\penalty0 (1):\penalty0 96, 2019.

\bibitem[Hou et~al.(2022)Hou, Su, Fu, and Yu]{HouSFY22}
Hou, J., Su, M., Fu, A., and Yu, Y.
\newblock Verifiable privacy-preserving scheme based on vertical federated
  random forest.
\newblock \emph{{IEEE} Internet Things}, 9\penalty0 (22):\penalty0
  22158--22172, 2022.

\bibitem[Hu et~al.(2019)Hu, Niu, Yang, and Zhou]{FDML}
Hu, Y., Niu, D., Yang, J., and Zhou, S.
\newblock {FDML:} {A} collaborative machine learning framework for distributed
  features.
\newblock \emph{Proc. {ACM} Int. Conf. Knowl. Discov. Data Min.}, pp.\
  2232--2240, 2019.

\bibitem[Huang et~al.(2010)Huang, Horowitz, and Wei]{huang2010variable}
Huang, J., Horowitz, J.~L., and Wei, F.
\newblock Variable selection in nonparametric additive models.
\newblock \emph{The Annals of Statistics}, 38\penalty0 (4):\penalty0
  2282--2313, 2010.

\bibitem[Ji et~al.(1992)Ji, Koll{\'a}r, and Shiffman]{ji1992global}
Ji, S., Koll{\'a}r, J., and Shiffman, B.
\newblock A global {\l}ojasiewicz inequality for algebraic varieties.
\newblock \emph{Trans. American Math. Soc.}, 329\penalty0 (2):\penalty0
  813--818, 1992.

\bibitem[Johnson et~al.(2016)Johnson, Pollard, Shen, Lehman, Feng, Ghassemi,
  Moody, Szolovits, Anthony~Celi, and Mark]{johnson2016mimic}
Johnson, A.~E., Pollard, T.~J., Shen, L., Lehman, L.-w.~H., Feng, M., Ghassemi,
  M., Moody, B., Szolovits, P., Anthony~Celi, L., and Mark, R.~G.
\newblock Mimic-iii, a freely accessible critical care database.
\newblock \emph{Nature}, 2016.

\bibitem[Kairouz et~al.(2021)Kairouz, McMahan, Avent, Bellet, Bennis, Bhagoji,
  Bonawitz, Charles, Cormode, Cummings, D'Oliveira, Eichner, Rouayheb, Evans,
  Gardner, Garrett, Gasc{\'{o}}n, Ghazi, Gibbons, Gruteser, Harchaoui, He, He,
  Huo, Hutchinson, Hsu, Jaggi, Javidi, Joshi, Khodak, Kone{\v{c}}n{\'y},
  Korolova, Koushanfar, Koyejo, Lepoint, Liu, Mittal, Mohri, Nock,
  {\"{O}}zg{\"{u}}r, Pagh, Qi, Ramage, Raskar, Raykova, Song, Song, Stich, Sun,
  Suresh, Tram{\`{e}}r, Vepakomma, Wang, Xiong, Xu, Yang, Yu, Yu, and
  Zhao]{kairouz2019advances}
Kairouz, P., McMahan, H.~B., Avent, B., Bellet, A., Bennis, M., Bhagoji, A.~N.,
  Bonawitz, K.~A., Charles, Z., Cormode, G., Cummings, R., D'Oliveira, R.
  G.~L., Eichner, H., Rouayheb, S.~E., Evans, D., Gardner, J., Garrett, Z.,
  Gasc{\'{o}}n, A., Ghazi, B., Gibbons, P.~B., Gruteser, M., Harchaoui, Z., He,
  C., He, L., Huo, Z., Hutchinson, B., Hsu, J., Jaggi, M., Javidi, T., Joshi,
  G., Khodak, M., Kone{\v{c}}n{\'y}, J., Korolova, A., Koushanfar, F., Koyejo,
  S., Lepoint, T., Liu, Y., Mittal, P., Mohri, M., Nock, R., {\"{O}}zg{\"{u}}r,
  A., Pagh, R., Qi, H., Ramage, D., Raskar, R., Raykova, M., Song, D., Song,
  W., Stich, S.~U., Sun, Z., Suresh, A.~T., Tram{\`{e}}r, F., Vepakomma, P.,
  Wang, J., Xiong, L., Xu, Z., Yang, Q., Yu, F.~X., Yu, H., and Zhao, S.
\newblock Advances and open problems in federated learning.
\newblock \emph{Found. Trends Mach. Learn.}, 14\penalty0 (1-2):\penalty0
  1--210, 2021.
\newblock \doi{10.1561/2200000083}.

\bibitem[Li et~al.(2023)Li, Peng, Zhang, Huang, Guo, Yu, and Liu]{fedsdgfs}
Li, A., Peng, H., Zhang, L., Huang, J., Guo, Q., Yu, H., and Liu, Y.
\newblock {FedSDG-FS}: Efficient and secure feature selection for vertical
  federated learning.
\newblock In \emph{{IEEE} Int. Conf. Comp. Comm.}, 2023.

\bibitem[Qiu et~al.(2022)Qiu, Zhang, Ji, Pu, and Wang]{Qiu2022hashing}
Qiu, P., Zhang, X., Ji, S., Pu, Y., and Wang, T.
\newblock All you need is hashing: Defending against data reconstruction attack
  in vertical federated learning.
\newblock \emph{arXiv:2212.00325}, 2022.

\bibitem[Sun et~al.(2019)Sun, Ippel, van Soest, Wouters, Malic, Adekunle,
  van~den Berg, Mussmann, Koster, van~der Kallen, van Oppen, Townend, Dekker,
  and Dumontier]{Sun2019-ki}
Sun, C., Ippel, L., van Soest, J., Wouters, B., Malic, A., Adekunle, O.,
  van~den Berg, B., Mussmann, O., Koster, A., van~der Kallen, C., van Oppen,
  C., Townend, D., Dekker, A., and Dumontier, M.
\newblock A {Privacy-Preserving} infrastructure for analyzing personal health
  data in a vertically partitioned scenario.
\newblock 264:\penalty0 373--377, 2019.

\bibitem[Wang et~al.(2021)Wang, Zhang, Wang, Pu, and Pal]{0010ZWPP21}
Wang, J., Zhang, H., Wang, J., Pu, Y., and Pal, N.~R.
\newblock Feature selection using a neural network with group lasso
  regularization and controlled redundancy.
\newblock \emph{{IEEE} Trans. Neural Networks Learn. Syst.}, 32\penalty0
  (3):\penalty0 1110--1123, 2021.

\bibitem[Wu \& Liu(2009)Wu and Liu]{wu2009variable}
Wu, Y. and Liu, Y.
\newblock Variable selection in quantile regression.
\newblock \emph{Statistica Sinica}, pp.\  801--817, 2009.

\bibitem[Yang et~al.(2019)Yang, Liu, Chen, and
  Tong]{DBLP:journals/tist/YangLCT19}
Yang, Q., Liu, Y., Chen, T., and Tong, Y.
\newblock Federated machine learning: Concept and applications.
\newblock \emph{{ACM} Trans. Intell. Syst. Technol.}, 10\penalty0 (2):\penalty0
  12:1--12:19, 2019.

\bibitem[Zhang et~al.(2020)Zhang, Wang, Sun, Zurada, and Pal]{ZhangWSZP20}
Zhang, H., Wang, J., Sun, Z., Zurada, J.~M., and Pal, N.~R.
\newblock Feature selection for neural networks using group lasso
  regularization.
\newblock \emph{{IEEE} Trans. Knowl. Data Eng.}, 32\penalty0 (4):\penalty0
  659--673, 2020.

\bibitem[Zhang et~al.(2022{\natexlab{a}})Zhang, Li, Hao, Chen, and
  Zhang]{ZhangLHCZ22}
Zhang, R., Li, H., Hao, M., Chen, H., and Zhang, Y.
\newblock Secure feature selection for vertical federated learning in ehealth
  systems.
\newblock In \emph{{IEEE} Int. Conf. Comm.}, pp.\  1257--1262. {IEEE},
  2022{\natexlab{a}}.

\bibitem[Zhang et~al.(2022{\natexlab{b}})Zhang, Hu, Gao, Gong, Guo, Gao, and
  Zhang]{zhang2022embedded}
Zhang, Y., Hu, Y., Gao, X., Gong, D., Guo, Y., Gao, K., and Zhang, W.
\newblock An embedded vertical-federated feature selection algorithm based on
  particle swarm optimisation.
\newblock \emph{CAAI Trans. Intel. Techn.}, 2022{\natexlab{b}}.

\bibitem[Zhao et~al.(2015)Zhao, Hu, and Wang]{ZhaoHW15}
Zhao, L., Hu, Q., and Wang, W.
\newblock Heterogeneous feature selection with multi-modal deep neural networks
  and sparse group {LASSO}.
\newblock \emph{{IEEE} Trans. Multim.}, 17\penalty0 (11):\penalty0 1936--1948,
  2015.

\bibitem[Zhou et~al.(2021)Zhou, Zhang, Lu, Dai, Chen, Liu, Pirttikangas, Shi,
  Zhang, and Herrera{-}Viedma]{zhou2021survey}
Zhou, J., Zhang, S., Lu, Q., Dai, W., Chen, M., Liu, X., Pirttikangas, S., Shi,
  Y., Zhang, W., and Herrera{-}Viedma, E.
\newblock A survey on federated learning and its applications for accelerating
  industrial internet of things.
\newblock \emph{arXiv:2104.10501}, 2021.

\bibitem[Zou et~al.(2022)Zou, Liu, Kang, Liu, He, Yi, Yang, and
  Zhang]{zou2022defending}
Zou, T., Liu, Y., Kang, Y., Liu, W., He, Y., Yi, Z., Yang, Q., and Zhang, Y.-Q.
\newblock Defending batch-level label inference and replacement attacks in
  vertical federated learning.
\newblock \emph{{IEEE} Trans. Big Data}, 2022.

\end{thebibliography}
\bibliographystyle{icml2023}

\clearpage
\newpage
\appendix
\onecolumn

\section{Vertical Federated Learning Algorithm} \label{algapp.sec}

In Algorithm~\ref{vfl.alg}, we present pseudocode for standard
VFL training with neural networks~\citep{FDML, SplitNN}.

\begin{algorithm}[H]
    \begin{algorithmic}[1]
    \STATE {\textbf{Initialize:}} $\thetab_m^0$ for all parties $m$ and server model $\thetab_s^0$
    \FOR {$t \leftarrow 0, \ldots, T_1-1$}
        \STATE Randomly sample $\B \subset [N]$
        \FOR {$m \leftarrow 1, \ldots, M$ in parallel}
            \STATE Send $\h_m(\thetab_m^t ; \X^{(\B)}_m)$ to server %
        \ENDFOR
        \STATE $\Phi \leftarrow 
            \{\thetab_s^t, \h_1(\thetab_1^t ; \X^{(\B)}_m), 
            \ldots, \h_M(\thetab_M^t ; \X^{(\B)}_m)\}$
        \STATE $\thetab_s^{t+1} \leftarrow U(\thetab_s^t, \nabla_s R_{\B}(\Phi^t ; \y^{\B^{t}}))$  %
        \STATE Server sends $\nabla_{\h_m(\thetab_m^t)} R_{\B}(\Phi^t ; \y^{(\B)})$ 
        to each party
        \FOR {$m \leftarrow 1, \ldots, M$ in parallel}
            \STATE $\nabla_m R_{\B}(\Phi^t) = \nabla_{\theta_m} \h_m(\thetab_m^t)^\top \nabla_{\h_m(\thetab_m^t)} R_{\B}(\Phi^t)$
            \STATE $\thetab_m^{t+1} = U(\thetab_m^t, \nabla_m R_{\B}(\Phi^t))$
        \ENDFOR
    \ENDFOR
    \end{algorithmic}
    \caption{Vertical Federated Learning}
    \label{vfl.alg}
\end{algorithm}

The parties start by agreeing upon
a mini-batch samples $\B$, then sending their current
embeddings for the given mini-batch to the server.
We let $\X^{(\B)}$ and $\y^{(\B)}$ denote the training samples and labels
in the mini-batch, respectively.
The server updates its model using the
mini-batch partial derivative with respect to $\thetab_s$, denoted 
by $\nabla_s R_{\B}(\cdot)$, and some optimizer update rule 
$U(\cdot)$ (e.g. SGD, Adam, etc.).
The server then sends the partial derivatives with respect to
the party's embeddings. Each party $m$ then updates its model
using its mini-batch partial derivative, denoted by $\nabla_m R_{\B}(\cdot)$.

\section{Proof of Theorem~\ref{main_feature.thm}}

In this section, we start by proving Theorems~\ref{client.thm} and 
\ref{server.thm} for the case when $M=1$, extend the results to $M>1$ case,
and finally prove Theorem~\ref{main_feature.thm}.
We provide a summary of the notation used in this section in Table~\ref{notation.table}.

\begin{table}[t]
    \caption{Summary of notation.}
\label{notation.table}
\vskip 0.1in
\small
\centering
    {\renewcommand{\arraystretch}{1.2}
\begin{tabular}{|l|l|}
    \hline
        \textbf{Notation} & \textbf{Definitions} \\
    \hline
    $N$ & Number of training samples. \\
    \hline
    $M$ & Number of parties. \\
    \hline
    $\lambda_s$, $\lambda_m$ & The server's and party's regularization coefficients, respectively. \\
    \hline
    $f(\cdot)$ & VFL model label prediction. \\
    \hline
    $\h(\cdot)$ & Party's local embedding function. \\
    \hline
    $e(\cdot)$ & Mean squared error between two embeddings. \\
    \hline
    $R(\cdot)$ & Risk function: MSE with labels for all possible samples. \\ %
    \hline
    $R_N(\cdot)$ & Empirical risk function: MSE with labels for all training samples. \\ %
    \hline
    $G(\cdot)$ & Group lasso $L_{2,1}$ regularization term. \\
    \hline
    $H(\cdot)$ & Expected mean-squared difference between two embedding functions. \\
    \hline
    $H_N(\cdot)$ & Empirical mean-squared difference between two embedding functions. \\
    \hline
    $d(\cdot)$ & Distance between a vector and a set of vectors. \\
    \hline
    $\bm\Theta^{\true}$ & The generating model parameters defined in Assumption~\ref{data.assum}. \\
    \hline
    $\hat{\bm\Theta}$ & Pre-trained model parameters from minimizing empirical risk. \\
    \hline
    $\tilde{\bm\Theta}$ & Model with the same risk as the generating model closest to pre-trained model. \\
    \hline
    $\bar{\bm\Theta}$ & Learned model parameters after running LESS-VFL. \\
    \hline
    $\U_m,\V_m$ & Input weights on significant and non-significant features in party $m$'s generating model. \\
    \hline
    $\mathcal{T}^*$ & Set of models that have the same risk as the generating model. \\
    \hline
    $\mathcal{S}^*$ & Set of server models that have the same risk as the generating server model. \\
    \hline
    $\mathcal{C}^*$ & Set of client models that have the same risk as the generating client model. \\
    \hline
    $\mathcal{T}_{\phi}^*$ & Subset of $\mathcal{T}^*$ with non-significant feature weights set to zero. \\
    \hline
    $\mathcal{S}_{\phi}^*$ & Server models in $\mathcal{T}^*$ with non-significant embedding weights set to zero. \\
    \hline
    $\mathcal{C}_{\phi}^*$ & Party models in $\mathcal{T}^*$ with non-significant feature weights set to zero. \\
    \hline
\end{tabular}%
}
\end{table}

\subsection{Additional Notation for $M=1$}
We start by providing additional notation for proving Theorems~\ref{client.thm} and \ref{server.thm}.
We define the set of party and server model parameters that are in 
the optimal parameter set $\mathcal{T}^*$: 
$$\mathcal{C}^* = \{\thetab_m : \exists \thetab_s \text{ s.t. }
[\thetab_s, \thetab_m] \in \mathcal{T}^*\}$$
and
$$\mathcal{S}^* = \{\thetab_s : \exists \thetab_m \text{ s.t. } 
[\thetab_s, \thetab_m] \in \mathcal{T}^*\}.$$

We use the following lemmas proven by~\citet{DinhH20Neurips}:
\begin{lemma} \label{Hstar.lemma}
    Let $\U$ and $\V$ be the significant and non-significant input layer weights in a generating model $\bm\Theta^{\true}$.
    Let $\phi(\bm\Theta)$ be the parameters $\bm\Theta$ with all non-significant input layer weights $\V$ set to zero.
    Under Assumption~\ref{data.assum},
    \begin{itemize} 
        \item There exists $c_0 > 0$ such that $\| \U^k \| \geq c_0$ for all $\bm\Theta \in \mathcal{T}^*$ and $k=1,\ldots,d^s$ (where $d^s$ is the number of significant features).
        \item If $\bm\Theta \in \mathcal{T}^*$, then parameters $\phi(\bm\Theta)$
        also belongs in $\mathcal{T}^*$.
    \end{itemize}
\end{lemma}
\begin{lemma} \label{lbound.lemma}
    There exist $c_2,\nu>0$ such that:
    $$c_2 d(\bm\Theta, \mathcal{T}^*)^{\nu} \leq R(\bm\Theta) - R(\bm\Theta^{\true})$$
    for all $\bm\Theta \in \mathcal{T}$.
\end{lemma}
\begin{lemma} \label{genbound.lemma}
    For any $\delta>0$, there exists $c_1(\delta)>0$ such that for all 
    $\bm\Theta \in \mathcal{T}$:
    $$|R_N(\bm\Theta) - R(\bm\Theta)| \leq c_1\frac{\log N}{\sqrt{N}}$$
    with probability $1-\delta$.
\end{lemma}
We also prove the following lemma:
 \begin{lemma} \label{pretrain.lemma}
Given a model $\hat{\bm\Theta}$ defined by \eqref{pretrain.eq},
for any $\delta > 0$, there exists $C_{\delta}(\delta) > 0$
and $N \geq N_0(\delta)$ such that:
     \begin{align}
     d(\hat{\bm\Theta}, \mathcal{T}^*) \leq C_\delta \frac{\log N}{\sqrt{N}}
 \end{align}
 with probability $1-\delta$.
\end{lemma}

\begin{proof}
Let $[\hat{\thetab}_m^{\top}, \hat{\thetab}_s^{\top}]^{\top} = \hat{\bm\Theta}$ be the party
and server model parameters after the pre-training step.
We define $\tilde{\bm\Theta} = \argmin_{\bm\Theta \in \mathcal{T}^*} \| \bm\Theta - \hat{\bm\Theta} \|_2 
= [\tilde{\thetab}_s^{\top}, \tilde{\thetab}_m^{\top}]^{\top}$
as the optimal model closest to the pre-trained model. 
By Lemmas~\ref{lbound.lemma} and \ref{genbound.lemma}, we have the following:
\begin{align}
    c_2 d(\hat{\bm\Theta}, \mathcal{T}^*)^{\mu} 
    &= c_2 \| \tilde{\bm\Theta} - \hat{\bm\Theta} \|_2^{\mu} \\
    &\leq  R(\hat{\thetab}_s ; \h(\hat{\thetab}_m)) - R(\tilde{\thetab}_s ; \h(\tilde{\thetab}_m))\\
    &\leq 2c_1\frac{\log N}{\sqrt{N}} + R_N(\hat{\thetab}_s ; \h(\hat{\thetab}_m)) - R_N(\tilde{\thetab}_s ; \h(\tilde{\thetab}_m))
\end{align}
Note that $R_N(\hat{\thetab}_s ; \h(\hat{\thetab}_m)) \leq R_N(\tilde{\thetab}_s ; \h(\tilde{\thetab}_m))$, thus:
\begin{align}
    c_2 d(\hat{\bm\Theta}, \mathcal{T}^*)^{\mu} 
    &\leq 2c_1 \frac{\log N}{\sqrt{N}} \\
    d(\hat{\bm\Theta}, \mathcal{T}^*) 
    &\leq \left(\frac{2c_1}{c_2} \frac{\log N}{\sqrt{N}}\right)^{1/\mu}
\end{align}
We let $\mu=1$.  
This completes the proof of Lemma~\ref{pretrain.lemma}.
\end{proof}

\subsection{Proof of Lemma~\ref{sig.lemma}} \label{lemmasig.sec}
\begin{proof}
Note that the minimization of $H(\thetab_m ; \tilde{\thetab}_m ; \mathcal{K}_m)$
causes $\h(\thetab_m ; \x)^k = \h(\tilde{\thetab}_m ; \x)^k$ for all significant embedding components $k \in \mathcal{K}_m$ and any input $\x$. 
By the definition of $\mathcal{T}^*$ and Definition~\ref{sig.def}, this means that $R(\thetab_s^{\true} ; \h(\thetab_m^{\true})) = R(\tilde{\thetab}_s ; \h(\tilde{\thetab}_m)) = R(\tilde{\thetab}_s ; \h(\thetab_m))$. Thus, $[\tilde{\thetab}_s, \thetab_m] \in \mathcal{T}^*$.

By Lemma~$3.1$ in \cite{DinhH20Neurips}, because $\tilde{\bm\Theta} \in \mathcal{T}^*$, 
$f(\tilde{\bm\Theta} ; \x ; y) = f(\bm\Theta^{\true} ; \x ; y)$
for all inputs $\x$. This means that the significant and non-significant 
features for $f(\bm\Theta^{\true})$ must be the same for $f(\tilde{\bm\Theta})$. 
Let $\tilde{\s}$ and $\tilde{\z}$ be the significant and non-significant features for $f(\tilde{\bm\Theta})$.
It must be the case that $\tilde{\s} = \s$ and $\tilde{\z} = \z$.

Let $j \in \tilde{\z}$ be a non-significant feature in $f(\tilde{\bm\Theta})$ and 
let $r \in \mathbb{R}$. 
Let $g^j(\x,s)$ be a function that replaces $\x^j$ with value $s$.
We know that for all $k \in \mathcal{K}_m$:
$$\h(\thetab_m ; g^j(\x, r))^k = \h(\tilde{\thetab}_m ; g^j(\x, r))^k = \h(\tilde{\thetab}_m ; \x)^k.$$ 
In fact, by Proposition~\ref{sig.prop},
all embedding components in $\mathcal{K}_m$ only depend on $\tilde{\s}$.
Since $\h(\thetab_m ; \cdot)^k$ for all $k \in \mathcal{K}_m$ is unaffected by features in $\tilde{\z}$,
this means the set of non-significant features 
for $e(\thetab_m ; \tilde{\thetab}_m ; \mathcal{K}_m)$ contains
the set of non-significant features for $f(\tilde{\bm\Theta})$: $\z_h \supseteq \tilde{\z}$.

Similarly, let $k$ be a significant feature for $f(\tilde{\bm\Theta})$.
By Proposition~\ref{sig.prop}, for all $k \in \mathcal{K}_m$:
$$\h(\thetab_m ; g^k(\x, r))^k = \h(\tilde{\thetab}_m ; g^k(\x, r))^k \neq \h(\tilde{\thetab}_m ; \x)^k$$ 
for some $r \in \mathbb{R}$.
This means that:
\begin{align*}
&\sum_{k \in \mathcal{K}_m}( \h(\thetab_m ; g^k(\x, r))^k - \h(\tilde{\thetab}_m ; \x)^k )^2 \neq 
\sum_{k \in \mathcal{K}_m}(\h(\thetab_m ; \x)^k - \h(\tilde{\thetab}_m ; \x)^k )^2    
\end{align*}
and we can say that significant features for $e(\thetab_m ; \tilde{\thetab}_m ; \mathcal{K}_s)$ contains of the set of significant features for $f(\tilde{\bm\Theta})$: $\s_h \supseteq \tilde{\s}$.
Therefore, $\s_h = \tilde{\s} = \s$ and $\z_h = \tilde{\z} = \z$. 
\end{proof}

\subsection{Proof of Theorem~\ref{server.thm}} \label{serverthm.sec}
Next, we prove that the server solving \eqref{server_lasso.eq} 
finds an optimal solution that also
sets the non-significant embedding component weights to zero.
We define $\bar{\thetab}_s$ as the server model parameters that solves \eqref{server_lasso.eq}.
We start by proving the following lemma:
\begin{lemma} \label{server.lemma}
Let $L$ be the Lipschitz constant for $f(\cdot)$.
Given a pre-trained model $\hat{\bm\Theta} = [\hat{\thetab}_m^{\top}, \hat{\thetab}_s^{\top}]^{\top}$ defined by \eqref{pretrain.eq},
let $\tilde{\bm\Theta} = \argmin_{\bm\Theta \in \mathcal{T}^*} \| \bm\Theta - \hat{\bm\Theta} \|_2 
= [\tilde{\thetab}_s^{\top}, \tilde{\thetab}_m^{\top}]^{\top}$ be the optimal model closest to the pre-trained model. 
For any $\delta > 0$, there exists $C_1(\delta), C_2(\delta), C_3(\delta), C_4(\delta), C_5 > 0$
and $N \geq N_0(\delta)$ such that:
\begin{align}
d(\bar{\thetab}_s, \mathcal{S}^*) \leq d(\bar{\thetab}_s, \{\tilde{\thetab}_s\}) 
\leq \left(C_1\frac{\log N}{\sqrt{N}} + C_2\lambda_s^{\nu / (\nu - 1)}  + C_3Ld(\hat{\bm\Theta}, \mathcal{T}^*) \right)^{1/\nu}
\end{align}
and the sum over the non-significant embedding component weights is
\begin{align}
\sum_{l} \|\V_s^{l}\|_2
    \leq C_4 \frac{\log N}{\lambda_s\sqrt{N}} 
    + \frac{2L}{\lambda_s}d(\hat{\bm\Theta}, \mathcal{T}^*) 
    + C_5 d(\bar{\thetab}_s, \{\tilde{\thetab}_s\}).
\end{align}
    with probability $1-\delta$.
\end{lemma}

\begin{proof}
Note that $\{\tilde{\thetab}_s\}$ is the zero-level set of 
the analytic function $E(\thetab_s) = R(\thetab_s ; \h(\tilde{\thetab}_m)) - R(\tilde{\thetab}_s ; \h(\tilde{\thetab}_m))$.
We can apply the Łojasiewicz inequality~\citep{ji1992global} as follows:
\begin{align}
    c_2 d(\bar{\thetab}_s, \mathcal{S}^*)^{\nu} 
    &\leq
    c_2 d(\bar{\thetab}_s, \{\tilde{\thetab}_s\})^{\nu} \\
    &= c_2 \| \tilde{\thetab}_s - \bar{\thetab}_s \|_2^{\nu} \\
    &\leq R(\bar{\thetab}_s ; \h(\tilde{\thetab}_m)) - R(\tilde{\thetab}_s ; \h(\tilde{\thetab}_m)) 
    \label{server_offset.eq}
\end{align}

Since $f(\cdot)$ is analytic, we know that the risk function is smooth. 
Let $L$ be the Lipschitz constant for $R(\cdot)$.  For any $\thetab_s$ we have:
\begin{align}
    \left| R(\thetab_s ; \h(\tilde{\thetab}_m)) - R(\thetab_s ; \h(\hat{\thetab}_m)) \right| \label{Lsmooth1.eq}
    &\leq L \| [\thetab_s, \tilde{\thetab}_m] - [\thetab_s, \hat{\thetab}_m] \|_2  \\
    &\leq L \| \tilde{\bm\Theta} - \hat{\bm\Theta} \|_2  \\
    &= L d(\hat{\bm\Theta}, \mathcal{T}^*). \label{Lsmooth2.eq}
\end{align}

Applying \eqref{Lsmooth2.eq} and Lemma~\ref{genbound.lemma} to \eqref{server_offset.eq} we have:
\begin{align}
    c_2 d(\bar{\thetab}_s, \mathcal{S}^*)^{\nu} 
    &\leq 2Ld(\hat{\bm\Theta}, \mathcal{T}^*) + R(\bar{\thetab}_s ; \h(\hat{\thetab}_m)) - R(\tilde{\thetab}_s ; \h(\hat{\thetab}_m)) \\
    &\leq 2c_1 \frac{\log N}{\sqrt{N}} + 2Ld(\hat{\bm\Theta}, \mathcal{T}^*) + R_N(\bar{\thetab}_s ; \h(\hat{\thetab}_m)) - R_N(\tilde{\thetab}_s ; \h(\hat{\thetab}_m))
    \label{server_def.eq}
\end{align}

By the definition of $\bar{\thetab}_s$ in \eqref{server_lasso.eq} we have:
\begin{align}
    R_N(\bar{\thetab}_s ; \h(\hat{\thetab}_m)) + \lambda_s G(\bar{\thetab}_s)
    &\leq R_N(\tilde{\thetab}_s ; \h(\hat{\thetab}_m)) + \lambda_s G(\tilde{\thetab}_s) \\
    R_N(\bar{\thetab}_s ; \h(\hat{\thetab}_m)) - R_N(\tilde{\thetab}_s ; \h(\hat{\thetab}_m))
    &\leq \lambda_s G(\tilde{\thetab}_s) - \lambda_s G(\bar{\thetab}_s)
    \label{use_server_def.eq}
\end{align}

Plugging \eqref{use_server_def.eq} into \eqref{server_def.eq}, and noting that regularizer $G(\cdot)$ is smooth, we have:
\begin{align}
    c_2 \| \tilde{\thetab}_s - \bar{\thetab}_s \|_2^{\nu} 
    &\leq 2c_1 \frac{\log N}{\sqrt{N}} + 2Ld(\hat{\bm\Theta}, \mathcal{T}^*)
        + \lambda_s(G(\tilde{\thetab}_s) - G(\bar{\thetab}_s)) \\
    &\leq 2c_1 \frac{\log N}{\sqrt{N}} + 2Ld(\hat{\bm\Theta}, \mathcal{T}^*) + \lambda_s C \|\tilde{\thetab}_s - \bar{\thetab}_s\|_2
    \label{need_young.eq}
\end{align}
where $C$ is the Lipschitz constant for $G(\cdot)$.

By Young's inequality, we have:
\begin{align}
    \lambda_s C \|\tilde{\thetab}_s - \bar{\thetab}_s\|_2
    &\leq \frac{1}{\nu}\left(\frac{(c_2\nu)^{1/\nu}}{2}\|\tilde{\thetab}_s - \bar{\thetab}_s\|_2 \right)^{\nu}  +
    \frac{\nu-1}{\nu}\left(\frac{2C}{(2c_2)^{1/\nu}}\lambda_s\right)^{\nu/(\nu-1)} \\
    &= \frac{c_2}{2}\|\tilde{\thetab}_s - \bar{\thetab}_s\|_2  +
    \frac{2(\nu-1)C^{\nu/(\nu-1)}}{\nu(c_2\nu)^{1/(\nu-1)}}\lambda_s^{\nu/(\nu-1)}.
    \label{young1.eq}
\end{align}

Let $C_0 = \frac{2(\nu-1)C^{\nu/(\nu-1)}}{\nu(c_2\nu)^{1/(\nu-1)}}$.
Plugging \eqref{young1.eq} into \eqref{need_young.eq} we have:
\begin{align}
    c_2 \| \tilde{\thetab}_s - \bar{\thetab}_s \|_2^{\nu} 
    &\leq 2c_1 \frac{\log N}{\sqrt{N}} + 2Ld(\hat{\bm\Theta}, \mathcal{T}^*) + C_0\lambda_s^{\nu / (\nu - 1)} +
    \frac{c_2}{2} \| \tilde{\thetab}_s - \bar{\thetab}_s \|_2^{\nu} 
    \\
    \frac{c_2}{2} \| \tilde{\thetab}_s - \bar{\thetab}_s \|_2^{\nu} 
    &\leq 2c_1 \frac{\log N}{\sqrt{N}} + 2Ld(\hat{\bm\Theta}, \mathcal{T}^*) + C_0\lambda_s^{\nu / (\nu - 1)} 
    \\
    \| \tilde{\thetab}_s - \bar{\thetab}_s \|_2
    &\leq \left(\frac{4c_1}{c_2}\frac{\log N}{\sqrt{N}} + \frac{2C_0}{c_2}\lambda_s^{\nu / (\nu - 1)}  + \frac{4L}{c_2}d(\hat{\bm\Theta}, \mathcal{T}^*) \right)^{1/\nu}
\end{align}

Note that $G(\cdot)$ can be rewritten as $G(\cdot) = \sum_{k} \|\U_s^{k}\|_2 + \sum_{l} \|\V_s^{l}\|_2$.
Let $K(\cdot) = \sum_{k} \|\U_s^{k}\|_2$ be the sum of significant embedding component weights in the regularizer $G(\cdot)$.
Let $\phi(\thetab_s)$ be the parameters $\thetab_s$ with all non-significant embedding
component weights $\V_s$ set to zero.
Note that $R(\tilde{\thetab}_s ; \h(\tilde{\thetab}_m)) \leq R(\bar{\thetab}_s ; \h(\tilde{\thetab}_m))$ since $[\tilde{\thetab}_s, \tilde{\thetab}_m] \in \mathcal{T}^*$.
Using the definition of $\bar{\thetab}_s$ and the smoothness of $K(\cdot)$, we have the following:
\begin{align}
    \lambda_s \sum_{l} \|\V_s^{l}\|_2
    &\leq R_N(\phi(\tilde{\thetab}_s) ; \h(\hat{\thetab}_m)) - R_N(\bar{\thetab}_s ; \h(\hat{\thetab}_m)) + \lambda_s (K(\phi(\tilde{\thetab}_s)) - K(\bar{\thetab}_s)) \\
    &\leq 2c_1 \frac{\log N}{\sqrt{N}} + 2Ld(\hat{\bm\Theta}, \mathcal{T}^*) 
        + R(\tilde{\thetab}_s ; \h(\tilde{\thetab}_m)) - R(\bar{\thetab}_s ; \h(\tilde{\thetab}_m)) + \lambda_s (K(\tilde{\thetab}_s) - K(\bar{\thetab}_s)) \\
    &\leq 2c_1 \frac{\log N}{\sqrt{N}} + 2Ld(\hat{\bm\Theta}, \mathcal{T}^*) 
        + \lambda_s C \|\tilde{\thetab}_s - \bar{\thetab}_s\|_2 \\
    \sum_{l} \|\V_s^{l}\|_2
    &\leq 2c_1 \frac{\log N}{\lambda_s\sqrt{N}} + \frac{2L}{\lambda_s}d(\hat{\bm\Theta}, \mathcal{T}^*) + C d(\bar{\thetab}_s, \{\tilde{\thetab}_s\}).
\end{align}
This completes the proof of Lemma~\ref{server.lemma}.
\end{proof}

To complete the proof of Theorem~\ref{server.thm}, we look at the distance
of $\bar{\thetab}_s$ from the set of parameters $\mathcal{S}_{\phi}^*$.
Note that by Lemma~\ref{Hstar.lemma}, $\phi(\tilde{\thetab}_s) \in \mathcal{S}_{\phi}^*$.
Let $\V_{\thetab_s} = \sum_{l} \|\V_s^{l}\|_2$ be the sum over 
non-significant embedding component weights in a model $\thetab_s$.
\begin{align}
    d(\bar{\thetab}_s, \mathcal{S}_{\phi}^*)
    &\leq \| \bar{\thetab}_s - \phi(\tilde{\thetab}_s) \|_2 \\
    &\leq \| \bar{\thetab}_s - \tilde{\thetab}_s \|_2 + \| \phi(\tilde{\thetab}_s) - \tilde{\thetab}_s \|_2 \\
    &\leq \| \bar{\thetab}_s - \tilde{\thetab}_s \|_2 + \| \V_{\tilde{\thetab}_s} \|_2 \\
    &\leq \| \bar{\thetab}_s - \tilde{\thetab}_s \|_2 + \| \V_{\tilde{\thetab}_s} +  \V_{\bar{\thetab}_s} - \V_{\bar{\thetab}_s} \|_2 \\
    &\leq \| \bar{\thetab}_s - \tilde{\thetab}_s \|_2 + \| \V_{\bar{\thetab}_s} \|_2 + \| \V_{\tilde{\thetab}_s} - \V_{\bar{\thetab}_s} \|_2 \\
    &\leq \| \bar{\thetab}_s - \tilde{\thetab}_s \|_2 + \| \V_{\bar{\thetab}_s} \|_2 + C\| \bar{\thetab}_s - \tilde{\thetab}_s \|_2
    \label{phi_dist1.eq}
\end{align}

The proof of Theorem~\ref{server.thm} is completed by combining
Lemma~\ref{server.lemma}, Lemma~\ref{pretrain.lemma}, and \eqref{phi_dist1.eq}.

\subsection{Proof of Theorem~\ref{client.thm}} \label{clientthm.sec}

Next we prove that the party solving \eqref{client_lasso.eq} finds the 
optimal solution and sets all non-significant input layer weights to zero.
Following the same proof of Lemma~\ref{genbound.lemma} given by~\citet{DinhH20Neurips}, we can prove the following lemma:
\begin{lemma} \label{genboundH.lemma}
    For any $\delta>0$, there exist $c_1(\delta)>0$ such that for all 
    $\thetab, \thetab'$ and sets $\mathcal{K}$:
    $$|H_N(\thetab ; \thetab' ; \mathcal{K}) - H(\thetab ; \thetab' ; \mathcal{K})| \leq c_1\frac{\log N}{\sqrt{N}}$$
    with probability $1-\delta$.
\end{lemma}

Let $\bar{\thetab}_m$ be the parameters that solve \eqref{client_lasso.eq}.
We prove the following lemma:
\begin{lemma} \label{client.lemma}
Let $B$ be the Lipschitz constant for $H(\cdot)$.
Let $[\tilde{\thetab}_s^{\top}, \tilde{\thetab}_m^{\top}]^{\top} = \argmin_{\bm\Theta \in \mathcal{T}^*} \| \bm\Theta - \hat{\bm\Theta} \|_2$
where $\hat{\bm\Theta}$ is the pre-trained model defined in \eqref{pretrain.eq}.
If $\mathcal{K}_m$ in \eqref{client_lasso.eq} is the set of
significant embedding components for $f(\tilde{\thetab}_s ; \h(\tilde{\thetab}_m))$,
for any $\delta > 0$, there exists $C_1(\delta), C_2(\delta), C_3(\delta), C_4(\delta), C_5 > 0$
and $N \geq N_0(\delta)$ such that:
\begin{align}
d(\bar{\thetab}_m, \mathcal{C}^*) \leq d(\bar{\thetab}_m, \{\tilde{\thetab}_m\}) 
    \leq \left(C_1 \frac{\log N}{\sqrt{N}} + C_2 B d(\hat{\thetab}_m, \mathcal{C}^*) + C_3 (\lambda_m)^{\nu / (\nu-1)}\right)^{1/\nu}
\end{align}
and the sum over the non-significant input layer weights is
\begin{align}
\sum_{l} \|\V_m^{l}\|_2
    &\leq C_4 \frac{\log N}{\lambda_m\sqrt{N}} 
    + \frac{2B}{\lambda_m} d(\hat{\thetab}_m, \mathcal{C}^*) 
        + C_5 d(\bar{\thetab}_m, \{\tilde{\thetab}_m\})
\end{align}
with probability $1-\delta$.
\end{lemma}

\begin{proof}
Note that $\{\tilde{\thetab}_m\}$ is the zero-level set of 
$H(\bar{\thetab}_m ; \tilde{\thetab}_m ; \mathcal{K}_m)$.
Since $H(\cdot)$ is analytic, we can apply the Łojasiewicz inequality as follows:
\begin{align}
    c_2 d(\bar{\thetab}_m, \mathcal{C}^*)^{\nu} 
    &\leq c_2 d(\bar{\thetab}_m, \{\tilde{\thetab}_m\})^{\nu} \\
    &\leq H(\bar{\thetab}_m ; \tilde{\thetab}_m) \\
    &= H(\bar{\thetab}_m ; \tilde{\thetab}_m) - H(\tilde{\thetab}_m ; \tilde{\thetab}_m)
    \label{client_offset.eq}
\end{align}
where \eqref{client_offset.eq} follows from that fact that $H(\thetab_m ; \thetab_m)=0$ for any $\thetab_m$.

Since $H(\cdot)$ is analytic, we know $H(\cdot)$ is smooth. 
Let $B$ be the Lipschitz constant for $H(\cdot)$. For any $\thetab_m$ we have:
\begin{align}
    | H(\thetab_m ; \hat{\thetab}_m) - H(\thetab_m ; \tilde{\thetab}_m) |
    \leq B \| \hat{\thetab}_m - \tilde{\thetab}_m \|_2\\
    \leq B d(\hat{\thetab}_m, \mathcal{C}^*)
    \label{smooth2.eq}
\end{align}

Applying \eqref{smooth2.eq} and Lemma~\ref{genboundH.lemma} to \eqref{client_offset.eq}:
\begin{align}
    c_2 d(\bar{\thetab}_m, \mathcal{C}^*)^{\nu} 
    &\leq 2B d(\hat{\thetab}_m, \mathcal{C}^*) + H(\bar{\thetab}_m ; \hat{\thetab}_m) - H(\tilde{\thetab}_m ; \hat{\thetab}_m)\\
    &\leq 2c_1 \frac{\log N}{\sqrt{N}} + 2B d(\hat{\thetab}_m, \mathcal{C}^*) 
        + H_N(\bar{\thetab}_m ; \hat{\thetab}_m) - H_N(\tilde{\thetab}_m ; \hat{\thetab}_m) 
        \label{need_def.eq}
\end{align}

By the definition of $\bar{\thetab}_m$ in \eqref{client_lasso.eq}:
\begin{align}
    H_N(\bar{\thetab}_m ; \hat{\thetab}_m) + \lambda_m G(\bar{\thetab}_m) 
    &\leq H_N(\tilde{\thetab}_m ; \hat{\thetab}_m) + \lambda_m G(\tilde{\thetab}_m) \\
    H_N(\bar{\thetab}_m ; \hat{\thetab}_m) - H_N(\tilde{\thetab}_m ; \hat{\thetab}_m)
    &\leq \lambda_m (G(\tilde{\thetab}_m) - G(\bar{\thetab}_m))
    \label{use_def.eq}
\end{align}

Plugging \eqref{use_def.eq} into \eqref{need_def.eq}:
\begin{align}
    c_2 d(\bar{\thetab}_m, \mathcal{C}^*)^{\nu} 
    &\leq 2c_1 \frac{\log N}{\sqrt{N}} + 2B d(\hat{\thetab}_m, \mathcal{C}^*) + \lambda_m (G(\tilde{\thetab}_m) - G(\bar{\thetab}_m)) \\ 
    &\leq 2c_1 \frac{\log N}{\sqrt{N}} + 2B d(\hat{\thetab}_m, \mathcal{C}^*) + \lambda_m C\|\tilde{\thetab}_m - \bar{\thetab}_m\|_2
    \label{need_young2.eq}
\end{align}
where $C$ is the Lipschitz constant for $G(\cdot)$.

By Young's inequality:
\begin{align}
    \lambda_m C\|\tilde{\thetab}_m - \bar{\thetab}_m\|_2
    &\leq \frac{1}{\nu}\left(\frac{(c_2\nu)^{1/\nu}}{2}
    \|\tilde{\thetab}_m - \bar{\thetab}_m\|_2 \right)^{\nu}  +
    \frac{\nu-1}{\nu}\left(\frac{2C}{(2c_2)^{1/\nu}}\lambda_s\right)^{\nu/(\nu-1)} \\
    &= \frac{c_2}{2}\|\tilde{\thetab}_m - \bar{\thetab}_m\|_2  +
    \frac{2(\nu-1)C^{\nu/(\nu-1)}}{\nu(c_2\nu)^{1/(\nu-1)}}\lambda_m^{\nu/(\nu-1)}.
    \label{young2.eq}
\end{align}

Let $C_0 = \frac{2(\nu-1)C^{\nu/(\nu-1)}}{\nu(c_2\nu)^{1/(\nu-1)}}$.
Applying \eqref{young2.eq} to \eqref{need_young2.eq} we have:
\begin{align}
    c_2 \| \tilde{\thetab}_m - \bar{\thetab}_m \|_2^{\nu} 
    &\leq 2c_1 \frac{\log N}{\sqrt{N}} + 2B d(\hat{\thetab}_m, \mathcal{C}^*) + C_0 (\lambda_m)^{\nu / (\nu-1)} + \frac{c_2}{2}\|\tilde{\thetab}_m - \bar{\thetab}_m\|_2^{\nu} \\
    \frac{c_2}{2} \| \tilde{\thetab}_m - \bar{\thetab}_m \|_2^{\nu} 
    &\leq 2c_1 \frac{\log N}{\sqrt{N}} + 2B d(\hat{\thetab}_m, \mathcal{C}^*) + C_0 (\lambda_m)^{\nu / (\nu-1)} \\
    \frac{c_2}{2} \| \tilde{\thetab}_m - \bar{\thetab}_m \|_2^{\nu} 
    &\leq 2c_1 \frac{\log N}{\sqrt{N}} + 2B d(\hat{\thetab}_m, \mathcal{C}^*) + C_0 (\lambda_m)^{\nu / (\nu-1)} \\
     d(\bar{\thetab}_m, \mathcal{C}^*)
     \leq  d(\bar{\thetab}_m, \{\tilde{\thetab}_m\})
    &\leq \left(\frac{4c_1}{c_2} \frac{\log N}{\sqrt{N}} + \frac{4B}{c_2} d(\hat{\thetab}_m, \mathcal{C}^*) + \frac{2C_0}{c_2} (\lambda_m)^{\nu / (\nu-1)}\right)^{1/\nu}
\end{align}

Note that $G(\cdot)$ can be rewritten as $G(\cdot) = \sum_{k} \|\U_m^{k}\|_2 + \sum_{l} \|\V_m^{l}\|_2$.
Let $K(\cdot) = \sum_{k} \|\U_m^{k}\|_2$ be the sum of significant input layer weights in the regularizer $G(\cdot)$.
Let $\phi(\thetab_m)$ be the parameters $\thetab_m$ with all non-significant input layer weights $\V_m$ set to zero.
Note that under our assumption that $\mathcal{K}_m$ only contains significant embedding components and Proposition~\ref{sig.prop},
$H(\phi(\tilde{\thetab}_m) ; \tilde{\thetab}_m ; \mathcal{K}_m) = H(\tilde{\thetab}_m ; \tilde{\thetab}_m ; \mathcal{K}_m)=0$,
because non-significant features have no effect on significant embedding components.
By the definition of $\bar{\thetab}_m$:

\begin{align}
    \lambda_m \sum_{l} \|\V_m^{l}\|_2
    &\leq  H_N(\phi(\tilde{\thetab}_m); \hat{\thetab}_m ; \mathcal{K}_m) - H_N(\bar{\thetab}_m ; \hat{\thetab}_m ; \mathcal{K}_m)
        + \lambda_m (K(\phi(\tilde{\thetab}_m)) - K(\bar{\thetab}_m)) \\
    &\leq 2c_1 \frac{\log N}{\sqrt{N}} + 2B d(\hat{\thetab}_m, \mathcal{C}^*) 
        + H(\phi(\tilde{\thetab}_m); \tilde{\thetab}_m ; \mathcal{K}_m) - H(\bar{\thetab}_m ; \tilde{\thetab}_m ; \mathcal{K}_m)
        + \lambda_m (K(\tilde{\thetab}_m) - K(\bar{\thetab}_m)) \\
    &\leq 2c_1 \frac{\log N}{\sqrt{N}} + 2B d(\hat{\thetab}_m, \mathcal{C}^*)
        + \lambda_m (K(\tilde{\thetab}_m) - K(\bar{\thetab}_m)) \\
    &\leq 2c_1 \frac{\log N}{\sqrt{N}} + 2B d(\hat{\thetab}_m, \mathcal{C}^*) 
        + \lambda_m C \|\tilde{\thetab}_m - \bar{\thetab}_m\|_2 \\
    \sum_{l} \|\V_m^{l}\|_2
    &\leq 2c_1 \frac{\log N}{\lambda_m\sqrt{N}} 
    + \frac{2B}{\lambda_m} d(\hat{\thetab}_m, \mathcal{C}^*) 
        + C d(\bar{\thetab}_m, \{\tilde{\thetab}_m\})
\end{align}

This completes the proof of Lemma~\ref{server.lemma}.
\end{proof}

To complete the proof of Theorem~\ref{client.thm}, we look at the distance
of $\bar{\thetab}_m$ from the set of parameters $\mathcal{C}_{\phi}^*$.
Note that by Lemma~\ref{Hstar.lemma}, $\phi(\tilde{\thetab}_m) \in \mathcal{C}_{\phi}^*$.
Let $\V_{\thetab_m} = \sum_{l} \|\V_m^{l}\|_2$ be the sum over 
non-significant feature weights in a model $\thetab_m$.
\begin{align}
    d(\bar{\thetab}_m, \mathcal{C}_{\phi}^*)
    &\leq \| \bar{\thetab}_m - \phi(\thetab'_s) \|_2 \\
    &\leq \| \bar{\thetab}_m - \tilde{\thetab}_m \|_2 + \| \phi(\tilde{\thetab}_m) - \tilde{\thetab}_m \|_2 \\
    &\leq \| \bar{\thetab}_m - \tilde{\thetab}_m \|_2 + \| \V_{\tilde{\thetab}_m} \|_2 \\
    &\leq \| \bar{\thetab}_m - \tilde{\thetab}_m \|_2 + \| \V_{\tilde{\thetab}_m} +  \V_{\bar{\thetab}_m} - \V_{\bar{\thetab}_m} \|_2 \\
    &\leq \| \bar{\thetab}_m - \tilde{\thetab}_m \|_2 + \| \V_{\bar{\thetab}_m} \|_2 + \| \V_{\tilde{\thetab}_m} - \V_{\bar{\thetab}_m} \|_2 \\
    &\leq \| \bar{\thetab}_m - \tilde{\thetab}_m \|_2 + \| \V_{\bar{\thetab}_m} \|_2 + C\| \bar{\thetab}_m - \tilde{\thetab}_m \|_2.
    \label{phi_dist2.eq}
\end{align}

Note that:
\begin{align}
    d(\hat{\thetab}_m, \mathcal{C}^*)
    = \| \tilde{\thetab}_m - \hat{\thetab}_m \|_2
    \leq
    \| \tilde{\bm\Theta} - \hat{\bm\Theta} \|_2
    =
    d(\hat{\bm\Theta}, \mathcal{T}^*).
    \label{HtoC.eq}
\end{align}

The proof of Theorem~\ref{client.thm} is completed by combining
Lemma~\ref{client.lemma}, \eqref{phi_dist2.eq}, \eqref{HtoC.eq}, and Lemma~\ref{pretrain.lemma}.

\subsection{Extension to $M>1$ Parties} \label{extension.sec}

\begin{proposition} \label{sig2.prop}
Consider a model $\bm\Theta = [\thetab_s^{\top}, \thetab_1^{\top}, \ldots, \thetab_M^{\top}]^{\top}$.
Let $\s$ and $\z$ be the sets of significant and non-significant features for $f(\bm\Theta)$, respectively.
Let the set of significant embedding components for $f(\thetab_s ; \h_1(\thetab_1^{\true}) ; \ldots ; \h_M(\thetab_M^{\true}))$ be $\s_s$.
Let $g^j(\x_m, r)$ replace input $\x_m^j$ with value $r$.
For each significant embedding component $k \in \s_s$, for all $j \in \z$ and $m \in [M]$, and any $r \in \mathbb{R}$, $\h_m(\thetab_m; \x_m)^k = \h_m(\thetab_m; g^j(\x_m, r))^k$.
\end{proposition}

\begin{proof}
Suppose that for a party $m$,
$\h_m(\thetab_m; \x_m)^k \neq \h_m(\thetab_m; g^j(\x_m, r))^k$
for some significant embedding component $k \in \s_s$, non-significant feature $j \in \z$, and $r \in \mathbb{R}$.
By our supposition and
since component $k$ is significant, 
$$f(\thetab_s ; \h_1(\thetab_1 ; \x_1) ; \ldots ; \h_m(\thetab_m ; \x_m ) ; \ldots ; \h_M(\thetab_M ; \x_M)) \neq f(\thetab_s ; \h_1(\thetab_1 ; \x_1) ; \ldots ; \h_m(\thetab_m ; g^j(\x_m ; r)) ; \ldots ; \h_M(\thetab_M ; \x_M))$$ 
for some value $r \in \mathbb{R}$. 
This contradicts the fact that $j$ is a non-significant feature.
\end{proof}

\begin{lemma} \label{sig2.lemma}
Let $\tilde{\bm\Theta} = [\tilde{\thetab}_s^{\top}, \tilde{\thetab}_1^{\top}, \ldots, \tilde{\thetab}_M^{\top}]^{\top} \in \mathcal{T}^*$.
Let $\s$ and $\z$ be the significant and non-significant features for $f(\bm\Theta^{\true})$.
Let $\mathcal{K}_m$ be the subset of significant embedding components for 
$f(\tilde{\thetab}_s ; \h_1(\tilde{\thetab}_1) ; \ldots ; \h_M(\tilde{\thetab}_M))$ in the embedding vector $\h_m(\tilde{\thetab}_m)$.
Let $\thetab_m = \argmin_{\thetab_m'} H(\thetab_m' ; \tilde{\thetab}_m ; \mathcal{K}_m)$
for all parties $m$.
Let $\s_{h_m}$ and $\z_{h_m}$ be the significant and non-significant features at each 
party $m$ for $e(\thetab_m ; \tilde{\thetab}_m ; \mathcal{K}_m)$ with parameters $\thetab_m$.
Then: $$[\tilde{\thetab}_s, \thetab_1, \ldots, \thetab_M] \in \mathcal{T}^*,~
\bigcup\limits_{m=1}^{M} \s_{h_m} = \s, \text{ and } \bigcup\limits_{m=1}^{M} \z_{h_m} = \z.$$
\end{lemma}

\begin{proof}
Note that the minimization of $H(\thetab_m ; \tilde{\thetab}_m ; \mathcal{K}_m)$
causes $\h_m(\thetab_m ; \x_m)_i = \h_m(\tilde{\thetab}_m ; \x_m)_i$ for all significant embedding components $i \in \mathcal{K}_m$ and any input $\x_m$. 
By the definition of $\mathcal{T}^*$ and Definition~\ref{sig.def}, this means that $$R(\thetab_s^{\true} ; \h_1(\thetab_1^{\true}) ; \ldots ; \h_M(\thetab_M^{\true})) = R(\tilde{\thetab}_s ; \h_1(\tilde{\thetab}_1) ; \ldots ; \h_M(\tilde{\thetab}_M)) = R(\tilde{\thetab}_s ; \h_1(\thetab_1) ; \ldots ; \h_M(\thetab_M)).$$ Thus, $[\tilde{\thetab}_s, \thetab_1, \ldots, \thetab_M] \in \mathcal{T}^*$.

By Lemma~$3.1$ in \cite{DinhH20Neurips}, because $\tilde{\bm\Theta} \in \mathcal{T}^*$, 
$f(\tilde{\bm\Theta} ; \x ; y) = f(\bm\Theta^{\true} ; \x ; y)$
for all inputs $\x$. This means that the significant and non-significant 
features for $f(\bm\Theta^{\true})$ must be the same for $f(\tilde{\bm\Theta})$. 
Let $\tilde{\s}$ and $\tilde{\z}$ be the significant and non-significant features for $f(\tilde{\bm\Theta})$.
It must be the case that $\tilde{\s} = \s$ and $\tilde{\z} = \z$.

Let $j \in \tilde{\z}_m$ be a non-significant feature for some party $m$ 
in $f(\tilde{\bm\Theta})$ and let $r \in \mathbb{R}$. 
Let $g^j(\cdot)$ be defined the same as in Definition~\ref{sig.def}.
We know that for all $k \in \mathcal{K}_m$:
$$\h_m(\thetab_m ; g^j(\x_m, r))^k = \h_m(\tilde{\thetab}_m ; g^j(\x_m, r))^k = \h_m(\tilde{\thetab}_m ; \x_m)^k$$ 
because by Proposition~\ref{sig.prop},
all embedding components in $\mathcal{K}_m$ only depend on $\tilde{\s}_m$.
Since $\h_m(\thetab_m ; \cdot)^k$ is unaffected by features in $\tilde{\z}_m$,
this means the set of non-significant features 
for $e(\thetab_m ; \tilde{\thetab}_m ; \mathcal{K}_m)$ contains
the set of non-significant features for $f(\tilde{\bm\Theta})$ at party $m$: $\z_{h_m} \supseteq \tilde{\z}_m$.

Similarly, let $k \in \tilde{\s}_m$ be a significant feature for some party $m$ in $f(\tilde{\bm\Theta})$.
By Proposition~\ref{sig2.prop}, for some $k \in \mathcal{K}_m$:
$$\h_m(\thetab_m ; g^k(\x_m, r))^k = \h_m(\tilde{\thetab}_m ; g^k(\x_m, r))^k \neq \h_m(\tilde{\thetab}_m ; \x)^k$$ 
for some $r \in \mathbb{R}$.
This means that:
\begin{align*}
\sum_{k \in \mathcal{K}_m}(\h_m(\thetab_m ; g^k(\x_m, r))^k - \h_m(\tilde{\thetab}_m ; \x_m)^k )^2 \neq 
\sum_{k \in \mathcal{K}_m} (\h_m(\thetab_m ; \x_m)^k - \h_m(\tilde{\thetab}_m ; \x_m)^k )^2    
\end{align*}
and we can say that significant features for $e(\thetab_m ; \tilde{\thetab}_m ; \mathcal{K}_m)$ contains the set of significant features for $f(\tilde{\bm\Theta})$ at party $m$: $\s_{h_m} \supseteq \tilde{\s}_m$.

Since for each party $m$, $\s_{h_m} \cap \z_{h_m} = \emptyset$, 
$\s_{h_m} = \tilde{\s}_m$ and $\z_{h_m} = \tilde{\z}_m$. 
Since for parties $m \neq j$, $\tilde{\s}_m \cap \s_j' = \emptyset$ and $\tilde{\z}_m \cap \z_j' = \emptyset$,
$\bigcup\limits_{m=1}^{M} \s_{h_m} = \tilde{\s} = \s$ and $\bigcup\limits_{m=1}^{M} \z_{h_m} = \tilde{\z} = \z$. 
This completes the proof of Lemma~\ref{sig2.lemma}.
\end{proof}

We redefine $\mathcal{S}_{\phi}^*$ for $M>1$:
$$\mathcal{S}_{\phi}^* = \{\thetab_s : \exists \thetab_m ~\forall m=1,\ldots,M \text{ s.t. } 
[\thetab_s, \thetab_1,\ldots,\thetab_M] \in \mathcal{T}^* \text{ and } \V_s = \zero \}.$$
We bound the distance $d(\thetab_s, \mathcal{S}_{\phi}^*)$ in the following theorem.
\begin{theorem} \label{serverM.thm}
Let $L$ be the Lipschitz constant for $f(\cdot)$.
Given a pre-trained model $\hat{\bm\Theta}$ defined by \eqref{pretrain.eq},
for any $\delta > 0$, there exists $C_N, C_{\delta}(\delta) > 0$
and $N \geq N_0(\delta)$ such that:
    \begin{align}
        d(\thetab_s, \mathcal{S}_{\phi}^*)
        \leq L C_{N}\frac{\log N}{\lambda_s\sqrt{N}}
    + L C_{\delta} \left(\frac{\log N}{\sqrt{N}} + \lambda_s^{\nu / (\nu - 1)} \right)^{1/\nu}
    \end{align}
    with probability $1-\delta$.
    If $\lambda_s \sim N^{-1/4}$, then with probability $1-\delta$ there exists $C(\delta)>0$ such that:
    \begin{align}
    d(\thetab_s, \mathcal{S}_{\phi}^*) \leq L C \left( \frac{\log N}{N} \right)^{\frac{1}{4(\nu-1)}}
    \label{serverthmM.eq}
    \end{align}
\end{theorem}

\begin{proof}
The proof of Theorem~\ref{serverM.thm} is the same as the proof of 
Theorem~\ref{server.thm} in Appendix~\ref{serverthm.sec}
when replacing $R(\thetab_s ; \h(\thetab_m))$
with $R(\thetab_s ; \h_1(\thetab_1) ; \ldots \h_M(\thetab_M))$.
\end{proof}

We define the set of party $m$ parameters in $\mathcal{T}^*$ that have the weights
on local non-significant features set to zero as:
$$\mathcal{C}_{m}^* = \{\thetab_m : \exists \thetab_s \text{ and } \thetab_j ~\forall j \neq m \text{ s.t. } 
[\thetab_s, \thetab_1, \ldots, \thetab_m, \ldots, \thetab_M] \in \mathcal{T}^* \text{ and } \V_m = \zero \}.$$
We bound the distance $d(\thetab_m, \mathcal{C}_{m}^*) \rightarrow 0$ in the following theorem.
\begin{theorem} \label{clientM.thm}
Let $B$ be the Lipschitz constant for $H(\cdot)$.
Let $[\tilde{\thetab}_s^{\top}, \tilde{\thetab}_1^{\top}, \ldots, \tilde{\thetab}_M^{\top}]^{\top} = \argmin_{\bm\Theta \in \mathcal{T}^*} \| \bm\Theta - \hat{\bm\Theta} \|_2$
where $\hat{\bm\Theta}$ is the pre-trained model defined in \eqref{pretrain.eq}.
If $\mathcal{K}_m$ in \eqref{client_lasso.eq} is the subset of
significant embedding components for $f(\tilde{\thetab}_s ; \h_1(\tilde{\thetab}_1) ; \ldots ; \h_M(\tilde{\thetab}_M))$ in $\h_m(\tilde{\thetab}_m)$, then for each party $m$, for any $\delta > 0$, there exists $C_N, C_{\delta}(\delta) > 0$
and $N \geq N_0(\delta)$ such that:
    \begin{align}
        d(\thetab_m, \mathcal{C}_{m}^*)
        \leq B C_N  \frac{\log N}{\lambda_m\sqrt{N}} 
        + C_{\delta}\left(B \frac{\log N}{\sqrt{N}} + (\lambda_m)^{\nu / (\nu-1)}\right)^{1/\nu}
    \end{align}
    with probability $1-\delta$.
    If $\lambda_m \sim N^{-1/4}$, then with probability $1-\delta$ there exists $C_m(\delta)>0$ such that:
    \begin{align}
    d(\thetab_m, \mathcal{C}_{m}^*) \leq B C_m \left( \frac{\log N}{N} \right)^{\frac{1}{4(\nu-1)}}
        \label{clientthmM.eq}
    \end{align}
\end{theorem}
Theorem~\ref{clientM.thm} follows from applying the proof of Theorem~\ref{client.thm} to each party $m$, replacing $\mathcal{C}_{\phi}^*$ with $\mathcal{C}_{m}^*$.

\subsection{Proof of Theorem~\ref{main_feature.thm}}
Let constant $C$ be defined as in Theorem~\ref{serverM.thm} and let
constant $C_m$ be defined the same as in Theorem~\ref{clientM.thm} for all parties $m$.
Let $B_m$ be the Lipschitz constant of $H(\cdot)$ at party $m$.
Then by Theorems~\ref{serverM.thm} and \ref{clientM.thm}, with probability $1-\delta$:
\begin{align}
    d(\bar{\bm\Theta}, \mathcal{T}_{\phi}^*) 
    &= \sqrt{d(\bar{\thetab}_s, \mathcal{S}_{\phi}^*)^2 + d(\bar{\thetab}_1, \mathcal{C}_{1}^{\true})^2 + \ldots + d(\bar{\thetab}_M, \mathcal{C}_{m}^*)^2} \\
    &\leq \sqrt{L^2 C^2 \left( \frac{\log N}{N} \right)^{\frac{1}{2(\nu-1)}} 
    + \left( \frac{\log N}{N} \right)^{\frac{1}{2(\nu-1)}} \sum_{m=1}^M B_m^2 C_m^2} \\
    &\leq \sqrt{\left(L^2 C^2 + \sum_{m=1}^M B_m^2 C_m^2\right)}  \left( \frac{\log N}{N} \right)^{\frac{1}{4(\nu-1)}}  \\
    &= O\left(\sqrt{M}  \left( \frac{\log N}{N} \right)^{\frac{1}{4(\nu-1)}} \right).
\end{align}

\section{Additional Experimental Results} \label{exp2.sec}

We now provide additional experimental results. 
We use the same experimental setup as described in Section~\ref{exp.sec}, and provide results for the datasets that were not included previously (MIMIC-III, Gina, Sylva).
We also include a complete results from the grid search,
showing the percentage of spurious feature removed and final training accuracy of group lasso, local lasso, and LESS-VFL with different regularization parameters.

In Figure~\ref{redapp.fig}, we plot the percentage of spurious features removed over $150$ communication epochs of training in the MIMIC-III, Gina, and Sylva datasets. For MIMIC-III and Sylva, we can see that all method perform similarly in terms of removing spurious features quickly, though group lasso lags behind the other methods by a few communication rounds.
In the case of Gina, group lasso takes about $20$ additional communication epochs to start removing spurious features, and oscillates before settling at a percentage lower than the other methods.
Reinforcing the takeaways from the main paper, by allowing feature selection to take place with minimal upfront communication, spurious features can be removed in fewer communication rounds compared to group lasso.

\begin{figure}[t]
    \begin{subfigure}{0.32\textwidth}
        \centering
        \includegraphics[width=\textwidth]{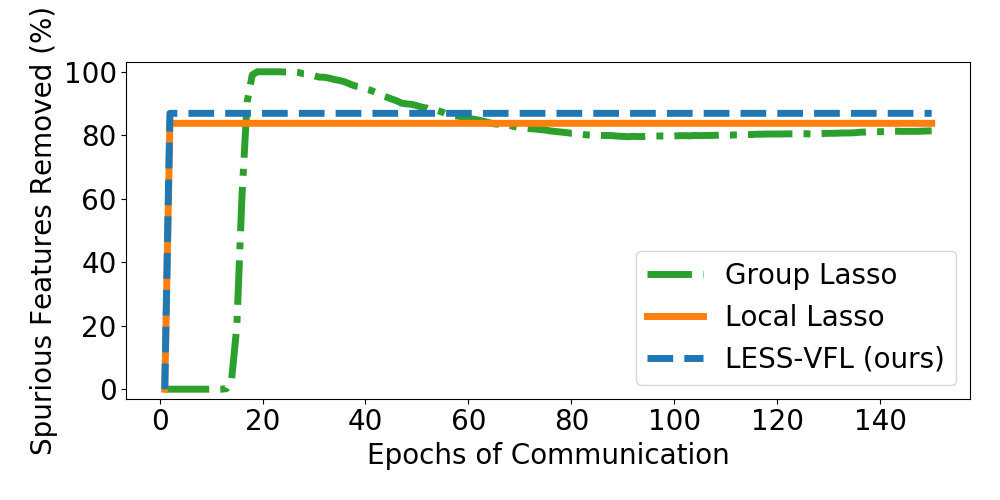}
        \caption{MIMIC-III}
        \label{mimicbar.fig}
    \end{subfigure}
    \hfill
    \begin{subfigure}{0.32\textwidth}
        \centering
        \includegraphics[width=\textwidth]{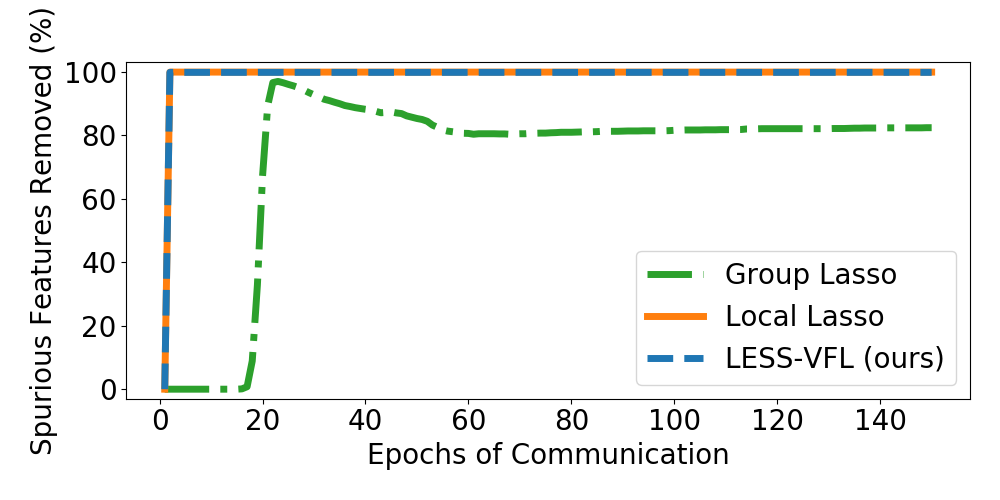}
        \caption{Gina}
        \label{ginabar.fig}
    \end{subfigure}
    \hfill
        \begin{subfigure}{0.32\textwidth}
        \centering
        \includegraphics[width=\textwidth]{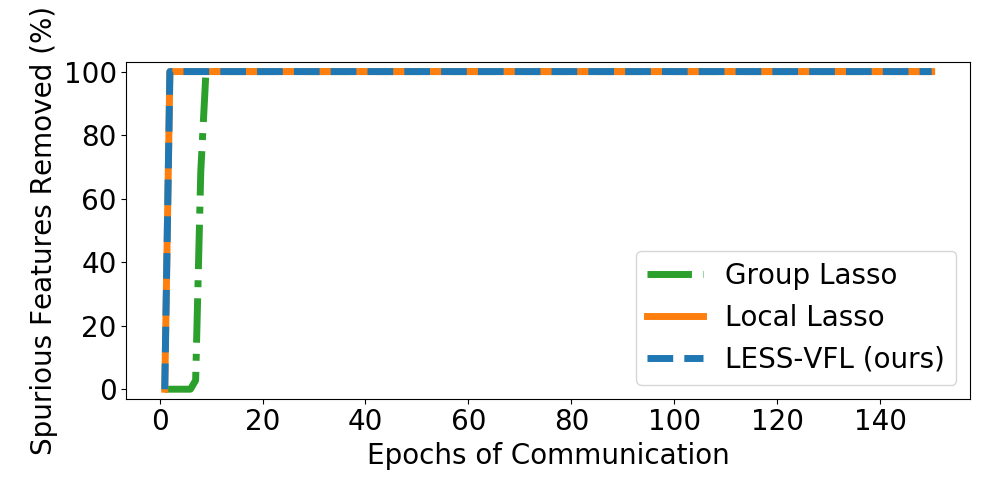}
        \caption{Sylva}
        \label{sylvabar.fig}
    \end{subfigure}
    \hfill
    \caption{Percentage of spurious features removed over $150$ communication epochs. The values shown is the average of $5$ runs. Group Lasso gradually removes features while local lasso and LESS-VFL remove features with one round of communication after pre-training.
    }
    \label{redapp.fig}
\end{figure}

In Figure~\ref{accapp.fig}, we plot the test accuracy against communication cost for all baselines. 
For both MIMIC-III and Sylva, the inclusion of spurious features does not have a large detrimental effect on the VFL test accuracy. In this case, it is important that applying the feature selection methods do not lead to model performance becoming worse than if we had not removed any spurious features.
In the case of MIMIC-III, all methods achieve similar test accuracy.
However, for the Sylva dataset, group lasso is unable to achieve the same accuracy as the other methods in the first $50$ communication epochs.
For the Gina dataset, all feature selection methods achieve test accuracy similar to the VFL baseline without spurious features, although group lasso takes more communication rounds to converge.

\begin{figure}[t]
    \begin{subfigure}{0.32\textwidth}
        \centering
        \includegraphics[width=\textwidth]{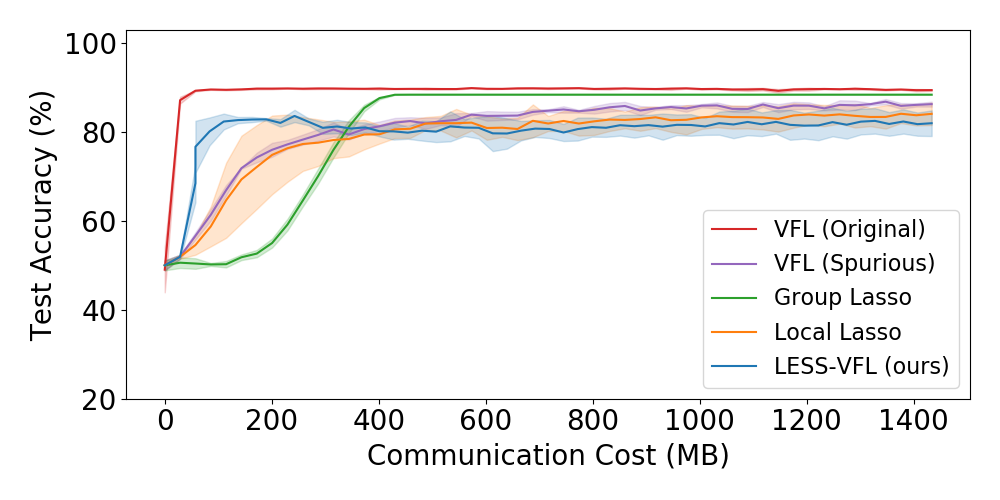}
        \caption{MIMIC-III}
        \label{mimicred5.fig}
    \end{subfigure}
    \hfill
        \begin{subfigure}{0.32\textwidth}
        \centering
        \includegraphics[width=\textwidth]{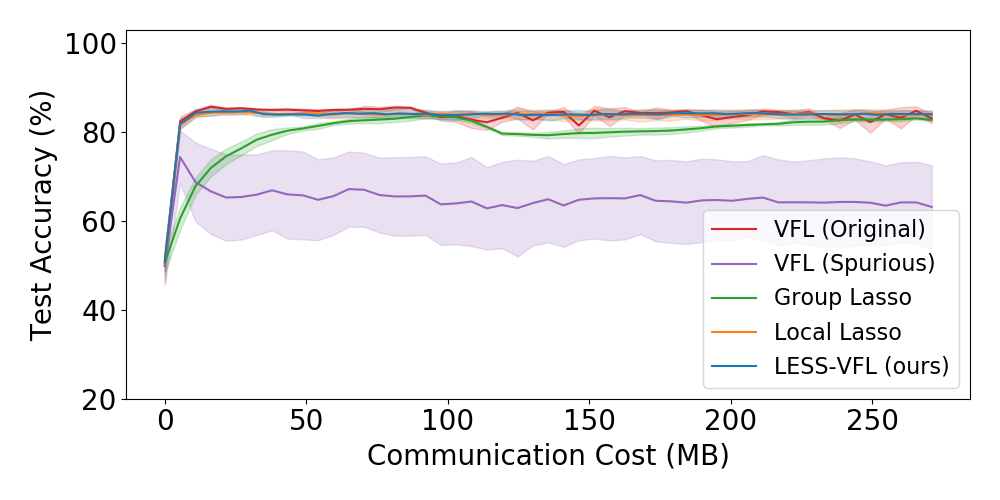}
        \caption{Gina}
        \label{ginared5.fig}
    \end{subfigure}
    \hfill
        \begin{subfigure}{0.32\textwidth}
        \centering
        \includegraphics[width=\textwidth]{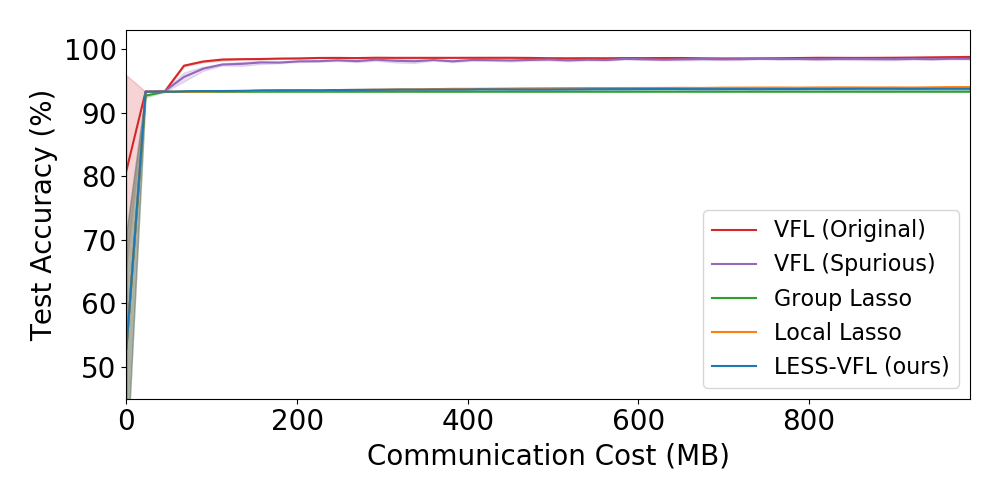}
        \caption{Sylva}
        \label{sylvared5.fig}
    \end{subfigure}
    \hfill
    \caption{Test accuracy over the first $50$ communication epochs. The solid lines are the average of $5$ runs and the shaded region represents the standard deviation.}
    \label{accapp.fig}
\end{figure}

\rev{
In Table~\ref{comm_cost2.table}, we provide the communication cost to reach $90\%$  of the baseline VFL (original) test accuracy and remove $80\%$ of spurious features for different amount of pre-training epochs. 
We show the communication cost taken during pre-training and post feature selection (Post-FS) as well as the total communication cost.
The values shown are the average of five runs, $\pm$ the standard deviation.
We can see that in many cases, LESS-VFL has zero cost for post feature selection. 
This indicates that LESS-VFL removed spurious features and achieves high accuracy 
during feature selection itself.
We can see that LESS-VFL always achieves the same or lower communication cost
than local lasso. Additionally, we see that in the Phishing dataset, local
lasso requires more pre-training epochs in order to achieve its lowest communication
cost to reach the thresholds.
In the Activity dataset, LESS-VFL always costs less communication than local lasso between $1$ and $5$ pre-training epochs. Local lasso's lowest communication cost is $26.56$ MB, while LESS-VFL's highest communication cost is $21.88$ MB.

\begin{table}
    \caption{Communication cost to reach $90\%$ of baseline VFL (original) test accuracy and remove $80\%$ of spurious features for different amount of pre-training epochs. All values are the average of $5$ runs. Bold values are the lowest communication cost achieved by that method on the dataset. `Pretrain' is the communication cost during pre-training, `Post-FS' is the communication cost during training after feature selection is complete, and `Total' is the sum of the previous.}
\label{comm_cost2.table}
\vskip 0.1in
\small
\centering
\begin{tabular}{lccccccc}
    \toprule 
    \multirow{2}{*}{\textbf{Dataset}}& \multirow{2}{*}{\begin{tabular}{@{}c@{}} \textbf{Pre-training} \\ \textbf{Epochs} \end{tabular}} & \multicolumn{6}{c}{\textbf{Communication Cost (MB)}} \\
    \cmidrule(rl){3-8}
    && \multicolumn{3}{c}{Local Lasso} & \multicolumn{3}{c}{LESS-VFL (ours)} \\
    \cmidrule(rl){3-5} \cmidrule(rl){6-8}
    && \subhead{Pretrain} & \subhead{Post-FS} & \subhead{Total} & \subhead{Pretrain} & \subhead{Post-FS} & \subhead{Total} \\
\midrule
\multirow{5}{*}{Activity} & $1$ & 3.59 & 30.15 & 33.74 & 3.59 & 15.97 & 19.56 \\
 & $2$ & 5.38 & 26.56 & 31.95 & 5.38 & 16.49 & 21.88 \\
 & $3$ & 7.18 & 19.39 & \textbf{26.56} & 7.18 & 10.90 & 18.08 \\
 & $4$ & 8.97 & 21.54 & 30.51 & 8.97 & 7.70 & \textbf{16.68} \\
 & $5$ & 10.77 & 18.85 & 29.62 & 10.77 & 10.51 & 21.28 \\
\midrule
\multirow{5}{*}{Phishing} & $1$ & 3.24 & 4.86 & 8.10 & 3.24 & 0.75 & \textbf{3.99} \\
 & $2$ & 4.86 & 1.62 & \textbf{6.48} & 4.86 & 0.38 & 5.23 \\
 & $3$ & 6.48 & 1.62 & 8.10 & 6.48 & 0.00 & 6.48 \\
 & $4$ & 8.10 & 1.62 & 9.72 & 8.10 & 0.00 & 8.10 \\
 & $5$ & 9.72 & 1.62 & 11.34 & 9.72 & 0.00 & 9.72 \\
\midrule
\multirow{5}{*}{MIMIC-III} & $1$ & 7.17 & 0.00 & \textbf{7.17} & 7.17 & 0.00 & \textbf{7.17} \\
 & $2$ & 10.75 & 0.00 & 10.75 & 10.75 & 0.00 & 10.75 \\
 & $3$ & 14.34 & 0.00 & 14.34 & 14.34 & 0.00 & 14.34 \\
 & $4$ & 17.92 & 0.00 & 17.92 & 17.92 & 0.00 & 17.92 \\
 & $5$ & 21.51 & 0.00 & 21.51 & 21.51 & 0.00 & 21.51 \\
\midrule
\multirow{5}{*}{Gina} & $1$ & 1.35 & 0.54 & \textbf{1.90} & 1.35 & 0.13 & \textbf{1.48} \\
 & $2$ & 2.03 & 0.00 & 2.03 & 2.03 & 0.00 & 2.03 \\
 & $3$ & 2.71 & 0.00 & 2.71 & 2.71 & 0.00 & 2.71 \\
 & $4$ & 3.39 & 0.00 & 3.39 & 3.39 & 0.00 & 3.39 \\
 & $5$ & 4.06 & 0.00 & 4.06 & 4.06 & 0.00 & 4.06 \\
\midrule
\multirow{5}{*}{Sylva} & $1$ & 5.62 & 0.00 & \textbf{5.62} & 5.62 & 0.00 & \textbf{5.62} \\
 & $2$ & 8.43 & 0.00 & 8.43 & 8.43 & 0.00 & 8.43 \\
 & $3$ & 11.25 & 0.00 & 11.25 & 11.25 & 0.00 & 11.25 \\
 & $4$ & 14.06 & 0.00 & 14.06 & 14.06 & 0.00 & 14.06 \\
 & $5$ & 16.87 & 0.00 & 16.87 & 16.87 & 0.00 & 16.87 \\
    \bottomrule
\end{tabular}
\end{table}}

In Tables~\ref{acc_train3.table}, \ref{acc_train1.table}, and \ref{acc_train2.table}, we provide the results of our grid search, used to determine the best regularization parameters for each method.
We provide the final training accuracy and percentage of spurious features removed for group lasso, local lasso, 
and LESS-VFL using different regularization values: $(\lambda_m, \lambda_s)$.
Note that the server regularization parameter $\lambda_s$ only applies to LESS-VFL.
The values shown are the average of five runs, $\pm$ the standard deviation.

\begin{table}
    \caption{Training accuracy and percentage of spurious features removed for the Activity and Phishing datasets.}
\label{acc_train3.table}
\vskip 0.1in
\small
\centering
\resizebox{0.98\textwidth}{!}{
\begin{tabular}{llcccccc}
    \toprule 
    \multirow{3}{*}{\textbf{Dataset}}& \multirow{3}{*}{\begin{tabular}{@{}c@{}} \textbf{Regularizer} \\ \textbf{Coefficients} \\ ($\lambda_m$, $\lambda_s$) \end{tabular}} & \multicolumn{2}{c}{\textbf{Group Lasso}} & \multicolumn{2}{c}{\textbf{Local Lasso}} & \multicolumn{2}{c}{\textbf{LESS-VFL (ours)}} \\
      \cmidrule(rl){3-4}  \cmidrule(rl){5-6} \cmidrule(rl){7-8}
    && \subhead{Final} & \subhead{Spurious Features}& \subhead{Final} & \subhead{Spurious Features}& \subhead{Final} & \subhead{Spurious Features} \\
    && \subhead{Accuracy} & \subhead{Removed}& \subhead{Accuracy} & \subhead{Removed}& \subhead{Accuracy} & \subhead{Removed}\\
\midrule
\multirow{7}{*}{Activity} & $(2.0,0.5)$ & 18.22 $\pm$ 0.00 & 100.00 $\pm$ 0.00 & 74.18 $\pm$ 0.77 & 100.00 $\pm$ 0.00 & 47.22 $\pm$ 2.49 & 100.00 $\pm$ 0.00 \\
                      & $(2.0,0.25)$ & 18.22 $\pm$ 0.00 & 100.00 $\pm$ 0.00 & 74.18 $\pm$ 0.77 & 100.00 $\pm$ 0.00 & 69.38 $\pm$ 1.75 & 100.00 $\pm$ 0.00 \\
                      & $(2.0,0.1)$ & 18.22 $\pm$ 0.00 & 100.00 $\pm$ 0.00 & 74.18 $\pm$ 0.77 & 100.00 $\pm$ 0.00 & 73.53 $\pm$ 0.71 & 100.00 $\pm$ 0.00 \\
                      & $(2.0,0.05)$ & 18.22 $\pm$ 0.00 & 100.00 $\pm$ 0.00 & 74.18 $\pm$ 0.77 & 100.00 $\pm$ 0.00 & 74.38 $\pm$ 0.61 & 100.00 $\pm$ 0.00 \\
                      & $(1.0,0.5)$ & 18.22 $\pm$ 0.00 & 100.00 $\pm$ 0.00 & 74.18 $\pm$ 0.77 & 100.00 $\pm$ 0.00 & 50.00 $\pm$ 0.29 & 100.00 $\pm$ 0.00 \\
                      & $(1.0,0.25)$ & 18.22 $\pm$ 0.00 & 100.00 $\pm$ 0.00 & 74.18 $\pm$ 0.77 & 100.00 $\pm$ 0.00 & 69.38 $\pm$ 1.75 & 100.00 $\pm$ 0.00 \\
                      & $(1.0,0.1)$ & 18.22 $\pm$ 0.00 & 100.00 $\pm$ 0.00 & 74.18 $\pm$ 0.77 & 100.00 $\pm$ 0.00 & 73.53 $\pm$ 0.71 & 100.00 $\pm$ 0.00 \\
                      & $(1.0,0.05)$ & 18.22 $\pm$ 0.00 & 100.00 $\pm$ 0.00 & 74.18 $\pm$ 0.77 & 100.00 $\pm$ 0.00 & 74.38 $\pm$ 0.61 & 100.00 $\pm$ 0.00 \\
                      & $(0.5,0.5)$ & 25.37 $\pm$ 8.79 & 100.00 $\pm$ 0.00 & 74.18 $\pm$ 0.77 & 100.00 $\pm$ 0.00 & 45.95 $\pm$ 8.65 & 99.46 $\pm$ 0.54 \\
                      & $(0.5,0.25)$ & 25.37 $\pm$ 8.79 & 100.00 $\pm$ 0.00 & 74.18 $\pm$ 0.77 & 100.00 $\pm$ 0.00 & 72.96 $\pm$ 2.34 & 100.00 $\pm$ 0.00 \\
                      & $(0.5,0.1)$ & 25.37 $\pm$ 8.79 & 100.00 $\pm$ 0.00 & 74.18 $\pm$ 0.77 & 100.00 $\pm$ 0.00 & 73.95 $\pm$ 1.02 & 100.00 $\pm$ 0.00 \\
                      & $(0.5,0.05)$ & 25.37 $\pm$ 8.79 & 100.00 $\pm$ 0.00 & 74.18 $\pm$ 0.77 & 100.00 $\pm$ 0.00 & 74.38 $\pm$ 0.61 & 100.00 $\pm$ 0.00 \\
                      & $(0.25,0.5)$ & 57.10 $\pm$ 1.74 & 100.00 $\pm$ 0.00 & 73.72 $\pm$ 4.78 & 100.00 $\pm$ 0.00 & 45.54 $\pm$ 7.09 & 91.90 $\pm$ 3.76 \\
                      & $(0.25,0.25)$ & 57.10 $\pm$ 1.74 & 100.00 $\pm$ 0.00 & 73.72 $\pm$ 4.78 & 100.00 $\pm$ 0.00 & 78.66 $\pm$ 3.40 & 100.00 $\pm$ 0.00 \\
                      & $(0.25,0.1)$ & 57.10 $\pm$ 1.74 & 100.00 $\pm$ 0.00 & 73.72 $\pm$ 4.78 & 100.00 $\pm$ 0.00 & 74.02 $\pm$ 7.81 & 100.00 $\pm$ 0.00 \\
                      & $(0.25,0.05)$ & 57.10 $\pm$ 1.74 & 100.00 $\pm$ 0.00 & 73.72 $\pm$ 4.78 & 100.00 $\pm$ 0.00 & 73.73 $\pm$ 5.01 & 100.00 $\pm$ 0.00 \\
                      & $(0.1,0.5)$ & 75.47 $\pm$ 1.93 & 88.93 $\pm$ 3.06 & 86.75 $\pm$ 2.04 & 100.00 $\pm$ 0.00 & 49.68 $\pm$ 4.03 & 59.05 $\pm$ 7.24 \\
                      & $(0.1,0.25)$ & 75.47 $\pm$ 1.93 & 88.93 $\pm$ 3.06 & 86.75 $\pm$ 2.04 & 100.00 $\pm$ 0.00 & 86.70 $\pm$ 3.13 & 87.93 $\pm$ 8.91 \\
                      & $(0.1,0.1)$ & 75.47 $\pm$ 1.93 & 88.93 $\pm$ 3.06 & 86.75 $\pm$ 2.04 & 100.00 $\pm$ 0.00 & 88.31 $\pm$ 0.74 & 99.64 $\pm$ 0.71 \\
                      & $(0.1,0.05)$ & 75.47 $\pm$ 1.93 & 88.93 $\pm$ 3.06 & 86.75 $\pm$ 2.04 & 100.00 $\pm$ 0.00 & 87.14 $\pm$ 1.86 & 99.93 $\pm$ 0.14 \\
                      & $(0.05,0.5)$ & 89.98 $\pm$ 2.60 & 1.71 $\pm$ 0.61 & 90.17 $\pm$ 2.02 & 0.64 $\pm$ 0.35 & 43.91 $\pm$ 5.10 & 0.48 $\pm$ 0.34 \\
                      & $(0.05,0.25)$ & 89.98 $\pm$ 2.60 & 1.71 $\pm$ 0.61 & 90.17 $\pm$ 2.02 & 0.64 $\pm$ 0.35 & 89.13 $\pm$ 2.36 & 0.50 $\pm$ 0.36 \\
                      & $(0.05,0.1)$ & 89.98 $\pm$ 2.60 & 1.71 $\pm$ 0.61 & 90.17 $\pm$ 2.02 & 0.64 $\pm$ 0.35 & 87.75 $\pm$ 2.18 & 0.57 $\pm$ 0.43 \\
                      & $(0.05,0.05)$ & 89.98 $\pm$ 2.60 & 1.71 $\pm$ 0.61 & 90.17 $\pm$ 2.02 & 0.64 $\pm$ 0.35 & 91.11 $\pm$ 1.58 & 0.64 $\pm$ 0.35 \\
                      & $(0.01,0.5)$ & 89.05 $\pm$ 2.24 & 0.00 $\pm$ 0.00 & 90.45 $\pm$ 1.12 & 0.00 $\pm$ 0.00 & 43.91 $\pm$ 5.10 & 0.00 $\pm$ 0.00 \\
                      & $(0.01,0.25)$ & 89.05 $\pm$ 2.24 & 0.00 $\pm$ 0.00 & 90.45 $\pm$ 1.12 & 0.00 $\pm$ 0.00 & 90.15 $\pm$ 1.86 & 0.00 $\pm$ 0.00 \\
                      & $(0.01,0.1)$ & 89.05 $\pm$ 2.24 & 0.00 $\pm$ 0.00 & 90.45 $\pm$ 1.12 & 0.00 $\pm$ 0.00 & 89.39 $\pm$ 1.61 & 0.00 $\pm$ 0.00 \\
                      & $(0.01,0.05)$ & 89.05 $\pm$ 2.24 & 0.00 $\pm$ 0.00 & 90.45 $\pm$ 1.12 & 0.00 $\pm$ 0.00 & 87.97 $\pm$ 2.30 & 0.00 $\pm$ 0.00 \\
\midrule
\multirow{7}{*}{Phishing} & $(2.0,0.01)$ & 55.63 $\pm$ 0.00 & 100.00 $\pm$ 0.00 & 90.27 $\pm$ 0.67 & 100.00 $\pm$ 0.00 & 53.38 $\pm$ 4.50 & 100.00 $\pm$ 0.00 \\
                      & $(2.0,0.005)$ & 55.63 $\pm$ 0.00 & 100.00 $\pm$ 0.00 & 90.27 $\pm$ 0.67 & 100.00 $\pm$ 0.00 & 53.38 $\pm$ 4.50 & 100.00 $\pm$ 0.00 \\
                      & $(1.0,0.01)$ & 55.63 $\pm$ 0.00 & 100.00 $\pm$ 0.00 & 90.27 $\pm$ 0.67 & 100.00 $\pm$ 0.00 & 53.38 $\pm$ 4.50 & 100.00 $\pm$ 0.00 \\
                      & $(1.0,0.005)$ & 55.63 $\pm$ 0.00 & 100.00 $\pm$ 0.00 & 90.27 $\pm$ 0.67 & 100.00 $\pm$ 0.00 & 53.38 $\pm$ 4.50 & 100.00 $\pm$ 0.00 \\
                      & $(0.5,0.01)$ & 55.63 $\pm$ 0.00 & 100.00 $\pm$ 0.00 & 90.27 $\pm$ 0.67 & 100.00 $\pm$ 0.00 & 53.38 $\pm$ 4.50 & 100.00 $\pm$ 0.00 \\
                      & $(0.5,0.005)$ & 55.63 $\pm$ 0.00 & 100.00 $\pm$ 0.00 & 90.27 $\pm$ 0.67 & 100.00 $\pm$ 0.00 & 53.38 $\pm$ 4.50 & 100.00 $\pm$ 0.00 \\
                      & $(0.25,0.01)$ & 89.26 $\pm$ 1.96 & 100.00 $\pm$ 0.00 & 90.27 $\pm$ 0.67 & 100.00 $\pm$ 0.00 & 53.38 $\pm$ 4.50 & 100.00 $\pm$ 0.00 \\
                      & $(0.25,0.005)$ & 89.26 $\pm$ 1.96 & 100.00 $\pm$ 0.00 & 90.27 $\pm$ 0.67 & 100.00 $\pm$ 0.00 & 51.13 $\pm$ 5.52 & 100.00 $\pm$ 0.00 \\
                      & $(0.1,0.01)$ & 91.71 $\pm$ 0.20 & 84.00 $\pm$ 5.33 & 78.96 $\pm$ 15.00 & 92.00 $\pm$ 9.80 & 78.98 $\pm$ 15.01 & 90.67 $\pm$ 9.04 \\
                      & $(0.1,0.005)$ & 91.71 $\pm$ 0.20 & 84.00 $\pm$ 5.33 & 78.96 $\pm$ 15.00 & 92.00 $\pm$ 9.80 & 92.45 $\pm$ 0.00 & 93.33 $\pm$ 0.00 \\
                      & $(0.05,0.01)$ & 91.85 $\pm$ 0.43 & 0.00 $\pm$ 0.00 & 92.06 $\pm$ 0.11 & 0.00 $\pm$ 0.00 & 91.92 $\pm$ 0.11 & 0.00 $\pm$ 0.00 \\
                      & $(0.05,0.005)$ & 91.85 $\pm$ 0.43 & 0.00 $\pm$ 0.00 & 92.06 $\pm$ 0.11 & 0.00 $\pm$ 0.00 & 92.06 $\pm$ 0.12 & 0.00 $\pm$ 0.00 \\
                      & $(0.01,0.01)$ & 91.86 $\pm$ 0.21 & 0.00 $\pm$ 0.00 & 92.00 $\pm$ 0.11 & 0.00 $\pm$ 0.00 & 92.00 $\pm$ 0.12 & 0.00 $\pm$ 0.00 \\
                      & $(0.01,0.005)$ & 91.86 $\pm$ 0.21 & 0.00 $\pm$ 0.00 & 92.00 $\pm$ 0.11 & 0.00 $\pm$ 0.00 & 91.88 $\pm$ 0.26 & 0.00 $\pm$ 0.00 \\
\bottomrule
\end{tabular}}
\end{table}

\begin{table}
    \caption{Training accuracy and percentage of spurious features removed in MIMIC-III dataset.  A '--' means that the experiments with this regularization parameter choice was not run.}
\label{acc_train1.table}
\vskip 0.1in
\small
\centering
\resizebox{0.98\textwidth}{!}{
\begin{tabular}{llcccccc}
    \toprule 
    \multirow{3}{*}{\textbf{Dataset}}& \multirow{3}{*}{\begin{tabular}{@{}c@{}} \textbf{Regularizer} \\ \textbf{Coefficients} \\ ($\lambda_m$, $\lambda_s$) \end{tabular}} & \multicolumn{2}{c}{\textbf{Group Lasso}} & \multicolumn{2}{c}{\textbf{Local Lasso}} & \multicolumn{2}{c}{\textbf{LESS-VFL (ours)}} \\
      \cmidrule(rl){3-4}  \cmidrule(rl){5-6} \cmidrule(rl){7-8}
    && \subhead{Final} & \subhead{Spurious Features}& \subhead{Final} & \subhead{Spurious Features}& \subhead{Final} & \subhead{Spurious Features} \\
    && \subhead{Accuracy} & \subhead{Removed}& \subhead{Accuracy} & \subhead{Removed}& \subhead{Accuracy} & \subhead{Removed}\\
\midrule
\multirow{19}{*}{MIMIC-III} & $(40.0,0.5)$ & -- & -- & 81.97 $\pm$ 2.64 & 100.00 $\pm$ 0.00 & 81.12 $\pm$ 2.08 & 99.89 $\pm$ 0.22 \\
                      & $(40.0,0.25)$ & -- & -- & 81.97 $\pm$ 2.64 & 100.00 $\pm$ 0.00 & 80.60 $\pm$ 1.59 & 100.00 $\pm$ 0.00 \\
                      & $(40.0,0.1)$ & -- & -- & 81.97 $\pm$ 2.64 & 100.00 $\pm$ 0.00 & 80.61 $\pm$ 2.62 & 100.00 $\pm$ 0.00 \\
                      & $(40.0,0.05)$ & -- & -- & 81.97 $\pm$ 2.64 & 100.00 $\pm$ 0.00 & 80.82 $\pm$ 1.44 & 99.21 $\pm$ 1.57 \\
                      & $(35.0,0.5)$ & -- & -- & 80.94 $\pm$ 3.45 & 98.60 $\pm$ 2.81 & 81.28 $\pm$ 1.89 & 96.07 $\pm$ 5.68 \\
                      & $(35.0,0.25)$ & -- & -- & 80.94 $\pm$ 3.45 & 98.60 $\pm$ 2.81 & 80.96 $\pm$ 1.92 & 98.60 $\pm$ 2.81 \\
                      & $(35.0,0.1)$ & -- & -- & 80.94 $\pm$ 3.45 & 98.60 $\pm$ 2.81 & 81.79 $\pm$ 1.03 & 98.43 $\pm$ 3.15 \\
                      & $(35.0,0.05)$ & -- & -- & 80.94 $\pm$ 3.45 & 98.60 $\pm$ 2.81 & 80.23 $\pm$ 2.09 & 98.71 $\pm$ 2.58 \\
                      & $(32.5,0.1)$ & -- & -- & 83.45 $\pm$ 2.47 & 87.53 $\pm$ 10.81 & 81.85 $\pm$ 2.66 & 89.94 $\pm$ 10.14 \\
                      & $(32.5,0.05)$ & -- & -- & 83.45 $\pm$ 2.47 & 87.53 $\pm$ 10.81 & 82.56 $\pm$ 1.37 & 86.85 $\pm$ 10.21 \\
                      & $(30.0,0.5)$ & -- & -- & 84.45 $\pm$ 2.08 & 76.24 $\pm$ 6.35 & 84.39 $\pm$ 2.84 & 66.15 $\pm$ 13.21 \\
                      & $(30.0,0.25)$ & -- & -- & 84.45 $\pm$ 2.08 & 76.24 $\pm$ 6.35 & 84.26 $\pm$ 2.26 & 76.12 $\pm$ 4.97 \\
                      & $(30.0,0.1)$ & -- & -- & 84.45 $\pm$ 2.08 & 76.24 $\pm$ 6.35 & 85.21 $\pm$ 1.62 & 74.38 $\pm$ 5.37 \\
                      & $(30.0,0.05)$ & -- & -- & 84.45 $\pm$ 2.08 & 76.24 $\pm$ 6.35 & 83.61 $\pm$ 2.54 & 78.30 $\pm$ 5.57 \\
                      & $(25.0,0.5)$ & -- & -- & 85.62 $\pm$ 1.24 & 55.22 $\pm$ 5.03 & 87.53 $\pm$ 1.19 & 53.43 $\pm$ 3.23 \\
                      & $(25.0,0.25)$ & -- & -- & 85.62 $\pm$ 1.24 & 55.22 $\pm$ 5.03 & 87.82 $\pm$ 0.65 & 53.54 $\pm$ 1.39 \\
                      & $(25.0,0.1)$ & -- & -- & 85.62 $\pm$ 1.24 & 55.22 $\pm$ 5.03 & 87.21 $\pm$ 1.45 & 56.07 $\pm$ 5.01 \\
                      & $(25.0,0.05)$ & -- & -- & 85.62 $\pm$ 1.24 & 55.22 $\pm$ 5.03 & 85.73 $\pm$ 1.13 & 55.11 $\pm$ 5.03 \\
                      & $(20.0,0.5)$ & -- & -- & 85.51 $\pm$ 1.34 & 44.61 $\pm$ 2.23 & 87.58 $\pm$ 0.53 & 46.40 $\pm$ 1.90 \\
                      & $(20.0,0.25)$ & -- & -- & 85.51 $\pm$ 1.34 & 44.61 $\pm$ 2.23 & 88.01 $\pm$ 0.40 & 44.94 $\pm$ 1.51 \\
                      & $(20.0,0.1)$ & -- & -- & 85.51 $\pm$ 1.34 & 44.61 $\pm$ 2.23 & 87.78 $\pm$ 0.66 & 43.71 $\pm$ 1.59 \\
                      & $(20.0,0.05)$ & -- & -- & 85.51 $\pm$ 1.34 & 44.61 $\pm$ 2.23 & 87.50 $\pm$ 0.54 & 45.51 $\pm$ 1.48 \\
                      & $(15.0,0.5)$ & -- & -- & 85.19 $\pm$ 0.86 & 32.25 $\pm$ 5.01 & 87.39 $\pm$ 0.79 & 29.72 $\pm$ 3.10 \\
                      & $(15.0,0.25)$ & -- & -- & 85.19 $\pm$ 0.86 & 32.25 $\pm$ 5.01 & 86.92 $\pm$ 0.76 & 31.40 $\pm$ 4.34 \\
                      & $(15.0,0.1)$ & -- & -- & 85.19 $\pm$ 0.86 & 32.25 $\pm$ 5.01 & 86.67 $\pm$ 0.75 & 31.63 $\pm$ 3.76 \\
                      & $(15.0,0.05)$ & -- & -- & 85.19 $\pm$ 0.86 & 32.25 $\pm$ 5.01 & 86.48 $\pm$ 1.32 & 30.73 $\pm$ 2.98 \\
                      & $(10.0,0.5)$ & -- & -- & 83.92 $\pm$ 1.70 & 8.99 $\pm$ 2.20 & 86.37 $\pm$ 2.26 & 9.89 $\pm$ 1.66 \\
                      & $(10.0,0.25)$ & -- & -- & 83.92 $\pm$ 1.70 & 8.99 $\pm$ 2.20 & 87.50 $\pm$ 0.83 & 8.26 $\pm$ 1.52 \\
                      & $(10.0,0.1)$ & -- & -- & 83.92 $\pm$ 1.70 & 8.99 $\pm$ 2.20 & 84.82 $\pm$ 2.11 & 9.72 $\pm$ 3.01 \\
                      & $(10.0,0.05)$ & -- & -- & 83.92 $\pm$ 1.70 & 8.99 $\pm$ 2.20 & 85.65 $\pm$ 0.86 & 8.88 $\pm$ 1.99 \\
                      & $(2.0,0.5)$ & 80.21 $\pm$ 1.25 & 100.00 $\pm$ 0.00 & -- & -- & -- & -- \\
                      & $(2.0,0.25)$ & 80.21 $\pm$ 1.25 & 100.00 $\pm$ 0.00 & -- & -- & -- & -- \\
                      & $(2.0,0.1)$ & 80.21 $\pm$ 1.25 & 100.00 $\pm$ 0.00 & -- & -- & -- & -- \\
                      & $(2.0,0.05)$ & 80.21 $\pm$ 1.25 & 100.00 $\pm$ 0.00 & -- & -- & -- & -- \\
                      & $(1.0,0.5)$ & 83.88 $\pm$ 2.24 & 98.93 $\pm$ 0.63 & -- & -- & -- & -- \\
                      & $(1.0,0.25)$ & 83.88 $\pm$ 2.24 & 98.93 $\pm$ 0.63 & -- & -- & -- & -- \\
                      & $(1.0,0.1)$ & 83.88 $\pm$ 2.24 & 98.93 $\pm$ 0.63 & -- & -- & -- & -- \\
                      & $(1.0,0.05)$ & 83.88 $\pm$ 2.24 & 98.93 $\pm$ 0.63 & -- & -- & -- & -- \\
                      & $(0.5,0.5)$ & 87.55 $\pm$ 0.74 & 93.37 $\pm$ 0.98 & -- & -- & -- & -- \\
                      & $(0.5,0.25)$ & 87.55 $\pm$ 0.74 & 93.37 $\pm$ 0.98 & -- & -- & -- & -- \\
                      & $(0.5,0.1)$ & 87.55 $\pm$ 0.74 & 93.37 $\pm$ 0.98 & -- & -- & -- & -- \\
                      & $(0.5,0.05)$ & 87.55 $\pm$ 0.74 & 93.37 $\pm$ 0.98 & -- & -- & -- & -- \\
                      & $(0.25,0.5)$ & 87.64 $\pm$ 0.92 & 81.35 $\pm$ 3.42 & -- & -- & -- & -- \\
                      & $(0.25,0.25)$ & 87.64 $\pm$ 0.92 & 81.35 $\pm$ 3.42 & -- & -- & -- & -- \\
                      & $(0.25,0.1)$ & 87.64 $\pm$ 0.92 & 81.35 $\pm$ 3.42 & -- & -- & -- & -- \\
                      & $(0.25,0.05)$ & 87.64 $\pm$ 0.92 & 81.35 $\pm$ 3.42 & -- & -- & -- & -- \\
                      & $(0.1,0.5)$ & 87.63 $\pm$ 0.47 & 60.84 $\pm$ 4.70 & -- & -- & -- & -- \\
                      & $(0.1,0.25)$ & 87.63 $\pm$ 0.47 & 60.84 $\pm$ 4.70 & -- & -- & -- & -- \\
                      & $(0.1,0.1)$ & 87.63 $\pm$ 0.47 & 60.84 $\pm$ 4.70 & -- & -- & -- & -- \\
                      & $(0.1,0.05)$ & 87.63 $\pm$ 0.47 & 60.84 $\pm$ 4.70 & -- & -- & -- & -- \\
                      & $(0.05,0.5)$ & 86.01 $\pm$ 1.41 & 43.99 $\pm$ 7.25 & -- & -- & -- & -- \\
                      & $(0.05,0.25)$ & 86.01 $\pm$ 1.41 & 43.99 $\pm$ 7.25 & -- & -- & -- & -- \\
                      & $(0.05,0.1)$ & 86.01 $\pm$ 1.41 & 43.99 $\pm$ 7.25 & -- & -- & -- & -- \\
                      & $(0.05,0.05)$ & 86.01 $\pm$ 1.41 & 43.99 $\pm$ 7.25 & -- & -- & -- & -- \\
                      & $(0.01,0.5)$ & 85.02 $\pm$ 0.84 & 0.00 $\pm$ 0.00 & -- & -- & -- & -- \\
                      & $(0.01,0.25)$ & 85.02 $\pm$ 0.84 & 0.00 $\pm$ 0.00 & -- & -- & -- & -- \\
                      & $(0.01,0.1)$ & 85.02 $\pm$ 0.84 & 0.00 $\pm$ 0.00 & -- & -- & -- & -- \\
                      & $(0.01,0.05)$ & 85.02 $\pm$ 0.84 & 0.00 $\pm$ 0.00 & -- & -- & -- & -- \\
    \bottomrule
\end{tabular}}
\end{table}
    
    \begin{table}
    \caption{Training accuracy and percentage of spurious features removed for the Gina and Sylva datasets. A '--' means that the experiments with this regularization parameter choice was not run.}
\label{acc_train2.table}
\vskip 0.1in
\small
\centering
\resizebox{0.98\textwidth}{!}{
\begin{tabular}{llcccccc}
    \toprule 
    \multirow{3}{*}{\textbf{Dataset}}& \multirow{3}{*}{\begin{tabular}{@{}c@{}} \textbf{Regularizer} \\ \textbf{Coefficients} \\ ($\lambda_m$, $\lambda_s$) \end{tabular}} & \multicolumn{2}{c}{\textbf{Group Lasso}} & \multicolumn{2}{c}{\textbf{Local Lasso}} & \multicolumn{2}{c}{\textbf{LESS-VFL (ours)}} \\
      \cmidrule(rl){3-4}  \cmidrule(rl){5-6} \cmidrule(rl){7-8}
    && \subhead{Final} & \subhead{Spurious Features}& \subhead{Final} & \subhead{Spurious Features}& \subhead{Final} & \subhead{Spurious Features} \\
    && \subhead{Accuracy} & \subhead{Removed}& \subhead{Accuracy} & \subhead{Removed}& \subhead{Accuracy} & \subhead{Removed}\\
\midrule
\multirow{8}{*}{Gina} & $(2.0,0.1)$ & 50.43 $\pm$ 0.00 & 100.00 $\pm$ 0.00 & 83.46 $\pm$ 1.10 & 100.00 $\pm$ 0.00 & 80.84 $\pm$ 0.00 & 100.00 $\pm$ 0.00 \\
                      & $(2.0,0.05)$ & 50.43 $\pm$ 0.00 & 100.00 $\pm$ 0.00 & 83.46 $\pm$ 1.10 & 100.00 $\pm$ 0.00 & 83.00 $\pm$ 1.15 & 100.00 $\pm$ 0.00 \\
                      & $(2.0,0.01)$ & 50.43 $\pm$ 0.00 & 100.00 $\pm$ 0.00 & 83.46 $\pm$ 1.10 & 100.00 $\pm$ 0.00 & 83.46 $\pm$ 1.25 & 100.00 $\pm$ 0.00 \\
                      & $(2.0,0.005)$ & 50.43 $\pm$ 0.00 & 100.00 $\pm$ 0.00 & 83.46 $\pm$ 1.10 & 100.00 $\pm$ 0.00 & 83.37 $\pm$ 1.34 & 100.00 $\pm$ 0.00 \\
                      & $(1.0,0.1)$ & 50.43 $\pm$ 0.00 & 100.00 $\pm$ 0.00 & 83.46 $\pm$ 1.10 & 100.00 $\pm$ 0.00 & 80.98 $\pm$ 0.00 & 100.00 $\pm$ 0.00 \\
                      & $(1.0,0.05)$ & 50.43 $\pm$ 0.00 & 100.00 $\pm$ 0.00 & 83.46 $\pm$ 1.10 & 100.00 $\pm$ 0.00 & 83.00 $\pm$ 1.15 & 100.00 $\pm$ 0.00 \\
                      & $(1.0,0.01)$ & 50.43 $\pm$ 0.00 & 100.00 $\pm$ 0.00 & 83.46 $\pm$ 1.10 & 100.00 $\pm$ 0.00 & 83.46 $\pm$ 1.25 & 100.00 $\pm$ 0.00 \\
                      & $(1.0,0.005)$ & 50.43 $\pm$ 0.00 & 100.00 $\pm$ 0.00 & 83.46 $\pm$ 1.10 & 100.00 $\pm$ 0.00 & 83.37 $\pm$ 1.34 & 100.00 $\pm$ 0.00 \\
                      & $(0.5,0.1)$ & 50.43 $\pm$ 0.00 & 100.00 $\pm$ 0.00 & 83.46 $\pm$ 1.10 & 100.00 $\pm$ 0.00 & 80.98 $\pm$ 0.00 & 100.00 $\pm$ 0.00 \\
                      & $(0.5,0.05)$ & 50.43 $\pm$ 0.00 & 100.00 $\pm$ 0.00 & 83.46 $\pm$ 1.10 & 100.00 $\pm$ 0.00 & 83.00 $\pm$ 1.15 & 100.00 $\pm$ 0.00 \\
                      & $(0.5,0.01)$ & 50.43 $\pm$ 0.00 & 100.00 $\pm$ 0.00 & 83.46 $\pm$ 1.10 & 100.00 $\pm$ 0.00 & 83.46 $\pm$ 1.25 & 100.00 $\pm$ 0.00 \\
                      & $(0.5,0.005)$ & 50.43 $\pm$ 0.00 & 100.00 $\pm$ 0.00 & 83.46 $\pm$ 1.10 & 100.00 $\pm$ 0.00 & 83.37 $\pm$ 1.34 & 100.00 $\pm$ 0.00 \\
                      & $(0.25,0.1)$ & 50.43 $\pm$ 0.00 & 100.00 $\pm$ 0.00 & 83.46 $\pm$ 1.10 & 100.00 $\pm$ 0.00 & 80.98 $\pm$ 0.00 & 99.59 $\pm$ 0.00 \\
                      & $(0.25,0.05)$ & 50.43 $\pm$ 0.00 & 100.00 $\pm$ 0.00 & 83.46 $\pm$ 1.10 & 100.00 $\pm$ 0.00 & 83.00 $\pm$ 1.15 & 99.79 $\pm$ 0.21 \\
                      & $(0.25,0.01)$ & 50.43 $\pm$ 0.00 & 100.00 $\pm$ 0.00 & 83.46 $\pm$ 1.10 & 100.00 $\pm$ 0.00 & 83.46 $\pm$ 1.25 & 100.00 $\pm$ 0.00 \\
                      & $(0.25,0.005)$ & 50.43 $\pm$ 0.00 & 100.00 $\pm$ 0.00 & 83.46 $\pm$ 1.10 & 100.00 $\pm$ 0.00 & 83.37 $\pm$ 1.34 & 100.00 $\pm$ 0.00 \\
                      & $(0.1,0.1)$ & 81.16 $\pm$ 1.12 & 82.33 $\pm$ 0.47 & 83.46 $\pm$ 1.10 & 100.00 $\pm$ 0.00 & 81.99 $\pm$ 0.00 & 81.20 $\pm$ 0.00 \\
                      & $(0.1,0.05)$ & 81.16 $\pm$ 1.12 & 82.33 $\pm$ 0.47 & 83.46 $\pm$ 1.10 & 100.00 $\pm$ 0.00 & 83.57 $\pm$ 0.86 & 99.38 $\pm$ 0.62 \\
                      & $(0.1,0.01)$ & 81.16 $\pm$ 1.12 & 82.33 $\pm$ 0.47 & 83.46 $\pm$ 1.10 & 100.00 $\pm$ 0.00 & 83.46 $\pm$ 1.25 & 100.00 $\pm$ 0.00 \\
                      & $(0.1,0.005)$ & 81.16 $\pm$ 1.12 & 82.33 $\pm$ 0.47 & 83.46 $\pm$ 1.10 & 100.00 $\pm$ 0.00 & 83.37 $\pm$ 1.34 & 100.00 $\pm$ 0.00 \\
                      & $(0.05,0.1)$ & 82.71 $\pm$ 0.24 & 54.75 $\pm$ 3.12 & 83.29 $\pm$ 0.84 & 100.00 $\pm$ 0.00 & 81.99 $\pm$ 0.00 & 43.18 $\pm$ 0.00 \\
                      & $(0.05,0.05)$ & 82.71 $\pm$ 0.24 & 54.75 $\pm$ 3.12 & 83.29 $\pm$ 0.84 & 100.00 $\pm$ 0.00 & 83.93 $\pm$ 1.22 & 88.22 $\pm$ 5.17 \\
                      & $(0.05,0.01)$ & 82.71 $\pm$ 0.24 & 54.75 $\pm$ 3.12 & 83.29 $\pm$ 0.84 & 100.00 $\pm$ 0.00 & 83.54 $\pm$ 1.37 & 100.00 $\pm$ 0.00 \\
                      & $(0.05,0.005)$ & 82.71 $\pm$ 0.24 & 54.75 $\pm$ 3.12 & 83.29 $\pm$ 0.84 & 100.00 $\pm$ 0.00 & 83.34 $\pm$ 0.87 & 100.00 $\pm$ 0.00 \\
                      & $(0.025,0.1)$ & -- & -- & 84.03 $\pm$ 0.67 & 99.83 $\pm$ 0.15 & 81.99 $\pm$ 0.00 & 30.37 $\pm$ 0.00 \\
                      & $(0.025,0.05)$ & -- & -- & 84.03 $\pm$ 0.67 & 99.83 $\pm$ 0.15 & 83.57 $\pm$ 0.00 & 71.90 $\pm$ 0.00 \\
                      & $(0.025,0.01)$ & -- & -- & 84.03 $\pm$ 0.67 & 99.83 $\pm$ 0.15 & 80.40 $\pm$ 0.62 & 99.71 $\pm$ 0.31 \\
                      & $(0.025,0.005)$ & -- & -- & 84.03 $\pm$ 0.67 & 99.83 $\pm$ 0.15 & 84.15 $\pm$ 0.92 & 99.75 $\pm$ 0.15 \\
                      & $(0.01,0.1)$ & 80.40 $\pm$ 0.75 & 0.00 $\pm$ 0.00 & 81.12 $\pm$ 1.01 & 0.00 $\pm$ 0.00 & 81.99 $\pm$ 0.00 & 0.00 $\pm$ 0.00 \\
                      & $(0.01,0.05)$ & 80.40 $\pm$ 0.75 & 0.00 $\pm$ 0.00 & 81.12 $\pm$ 1.01 & 0.00 $\pm$ 0.00 & 79.47 $\pm$ 0.94 & 0.00 $\pm$ 0.00 \\
                      & $(0.01,0.01)$ & 80.40 $\pm$ 0.75 & 0.00 $\pm$ 0.00 & 81.12 $\pm$ 1.01 & 0.00 $\pm$ 0.00 & 81.38 $\pm$ 1.61 & 0.00 $\pm$ 0.00 \\
                      & $(0.01,0.005)$ & 80.40 $\pm$ 0.75 & 0.00 $\pm$ 0.00 & 81.12 $\pm$ 1.01 & 0.00 $\pm$ 0.00 & 80.03 $\pm$ 1.89 & 0.00 $\pm$ 0.00 \\
\midrule
\multirow{7}{*}{Sylva} & $(2.0,0.01)$ & 93.30 $\pm$ 0.00 & 100.00 $\pm$ 0.00 & 97.60 $\pm$ 0.28 & 100.00 $\pm$ 0.00 & 97.55 $\pm$ 0.18 & 100.00 $\pm$ 0.00 \\
                      & $(2.0,0.005)$ & 93.30 $\pm$ 0.00 & 100.00 $\pm$ 0.00 & 97.60 $\pm$ 0.28 & 100.00 $\pm$ 0.00 & 97.58 $\pm$ 0.23 & 100.00 $\pm$ 0.00 \\
                      & $(1.0,0.01)$ & 93.30 $\pm$ 0.00 & 100.00 $\pm$ 0.00 & 97.60 $\pm$ 0.28 & 100.00 $\pm$ 0.00 & 97.55 $\pm$ 0.18 & 100.00 $\pm$ 0.00 \\
                      & $(1.0,0.005)$ & 93.30 $\pm$ 0.00 & 100.00 $\pm$ 0.00 & 97.60 $\pm$ 0.28 & 100.00 $\pm$ 0.00 & 97.58 $\pm$ 0.23 & 100.00 $\pm$ 0.00 \\
                      & $(0.5,0.01)$ & 93.30 $\pm$ 0.00 & 100.00 $\pm$ 0.00 & 97.60 $\pm$ 0.28 & 100.00 $\pm$ 0.00 & 97.55 $\pm$ 0.18 & 100.00 $\pm$ 0.00 \\
                      & $(0.5,0.005)$ & 93.30 $\pm$ 0.00 & 100.00 $\pm$ 0.00 & 97.60 $\pm$ 0.28 & 100.00 $\pm$ 0.00 & 97.58 $\pm$ 0.23 & 100.00 $\pm$ 0.00 \\
                      & $(0.25,0.01)$ & 93.30 $\pm$ 0.00 & 100.00 $\pm$ 0.00 & 97.60 $\pm$ 0.28 & 100.00 $\pm$ 0.00 & 97.55 $\pm$ 0.18 & 100.00 $\pm$ 0.00 \\
                      & $(0.25,0.005)$ & 93.30 $\pm$ 0.00 & 100.00 $\pm$ 0.00 & 97.60 $\pm$ 0.28 & 100.00 $\pm$ 0.00 & 97.58 $\pm$ 0.23 & 100.00 $\pm$ 0.00 \\
                      & $(0.1,0.01)$ & 93.30 $\pm$ 0.00 & 100.00 $\pm$ 0.00 & 98.54 $\pm$ 0.06 & 100.00 $\pm$ 0.00 & 98.61 $\pm$ 0.06 & 100.00 $\pm$ 0.00 \\
                      & $(0.1,0.005)$ & 93.30 $\pm$ 0.00 & 100.00 $\pm$ 0.00 & 98.54 $\pm$ 0.06 & 100.00 $\pm$ 0.00 & 98.55 $\pm$ 0.04 & 100.00 $\pm$ 0.00 \\
                      & $(0.05,0.01)$ & 98.52 $\pm$ 0.08 & 23.89 $\pm$ 4.32 & 98.56 $\pm$ 0.07 & 11.11 $\pm$ 2.03 & 98.53 $\pm$ 0.09 & 10.00 $\pm$ 1.08 \\
                      & $(0.05,0.005)$ & 98.52 $\pm$ 0.08 & 23.89 $\pm$ 4.32 & 98.56 $\pm$ 0.07 & 11.11 $\pm$ 2.03 & 98.59 $\pm$ 0.11 & 10.56 $\pm$ 1.81 \\
                      & $(0.01,0.01)$ & 98.46 $\pm$ 0.10 & 0.00 $\pm$ 0.00 & 98.37 $\pm$ 0.11 & 0.00 $\pm$ 0.00 & 98.49 $\pm$ 0.09 & 0.00 $\pm$ 0.00 \\
                      & $(0.01,0.005)$ & 98.46 $\pm$ 0.10 & 0.00 $\pm$ 0.00 & 98.37 $\pm$ 0.11 & 0.00 $\pm$ 0.00 & 98.53 $\pm$ 0.07 & 0.00 $\pm$ 0.00 \\
                      \bottomrule
\end{tabular}}
\end{table}

\end{document}